\documentclass{article}

\usepackage{minitoc}

\usepackage{microtype}
\usepackage{graphicx}
\usepackage{booktabs} %

\usepackage{hyperref}

\usepackage[accepted]{icml2024}

\usepackage{amsmath}
\usepackage{amssymb}
\usepackage{mathtools}
\usepackage{amsthm}

\usepackage[capitalize,noabbrev]{cleveref}

\theoremstyle{plain}
\newtheorem{theorem}{Theorem}[section]
\newtheorem{proposition}[theorem]{Proposition}
\newtheorem{lemma}[theorem]{Lemma}
\newtheorem{corollary}[theorem]{Corollary}
\theoremstyle{definition}
\newtheorem{definition}[theorem]{Definition}

\theoremstyle{remark}

\usepackage[textsize=tiny]{todonotes}

\icmltitlerunning{Bivariate Causal Discovery using Bayesian Model Selection}

\usepackage{graphicx} %
\usepackage{booktabs} %

\usepackage{amsmath}
\usepackage{amssymb}
\usepackage{mathtools}
\usepackage{amsthm}
\usepackage{tikz}
\usetikzlibrary{bayesnet}
\usetikzlibrary{arrows}
\usepackage{subcaption}
\usepackage{tikz}
\usepackage{tkz-tab}
\usepackage{bm}
\usepackage{mathtools}
\usepackage{multirow}
\usepackage{wrapfig}

\usepackage{hyperref}

\usepackage{soul}
\usepackage{cleveref}
\Crefname{appendix}{app.}{appendices}

\usepackage{MarkMathCmds}

\theoremstyle{plain}

\definecolor{causalcolour}{rgb}{0.00, 0.0, 1.0}

\newcommand{\model}{\mathcal{M}}
\newcommand{\xcausey}{\model_{\textbf{\color{causalcolour}X} \!\rightarrow\! Y}}
\newcommand{\ycausex}{\model_{X \!\leftarrow\! \textbf{\color{causalcolour}Y}}}
\newcommand{\nomodelxcausey}{\textbf{\color{causalcolour}X} \!\rightarrow\! Y}
\newcommand{\nomodelycausex}{X\!\leftarrow\!\textbf{\color{causalcolour}Y}}

\newcommand{\bfx}{\mathbf{x}}
\newcommand{\bfy}{\mathbf{y}}

\newcommand{\indep}{\perp \!\!\! \perp}

\begin{document}

\twocolumn[
\icmltitle{Bivariate Causal Discovery using Bayesian Model Selection}

\begin{icmlauthorlist}
\icmlauthor{Anish Dhir}{yyy}
\icmlauthor{Samuel Power}{yyy1}
\icmlauthor{Mark van der Wilk}{yyy2}
\end{icmlauthorlist}

\icmlaffiliation{yyy}{Imperial College London}
\icmlaffiliation{yyy1}{University of Bristol}
\icmlaffiliation{yyy2}{University of Oxford}

\icmlcorrespondingauthor{Anish Dhir}{anish.dhir13@imperial.ac.uk}

\icmlkeywords{Machine Learning, ICML}

\vskip 0.3in
]

\printAffiliationsAndNotice{}  %

\begin{abstract}
Much of the causal discovery literature prioritises guaranteeing the identifiability of causal direction in statistical models. For structures within a Markov equivalence class, this requires strong assumptions which may not hold in real-world datasets, ultimately limiting the usability of these methods. Building on previous attempts, we show how to incorporate causal assumptions within the Bayesian framework. Identifying causal direction then becomes a Bayesian model selection problem. This enables us to construct models with realistic assumptions, and consequently allows for the differentiation between Markov equivalent causal structures. We analyse why Bayesian model selection works in situations where methods based on maximum likelihood fail. To demonstrate our approach, we construct a Bayesian non-parametric model that can flexibly model the joint distribution. We then outperform previous methods on a wide range of benchmark datasets with varying data generating assumptions.
\end{abstract}

\begin{figure*}[t]
\begin{subfigure}{0.33\textwidth}
\centering

\resizebox{\linewidth}{!}{

\tikzset{every picture/.style={line width=2.5pt}} %

\begin{tikzpicture}[x=0.75pt,y=0.75pt,yscale=-1,xscale=1]

\draw    (50,180) -- (272,179.61) -- (600,180) ;
\draw [color={rgb, 255:red, 208; green, 2; blue, 27 }  ,draw opacity=1 ][line width=2.25]    (50,110) .. controls (107,115) and (175,55) .. (290,100) .. controls (405,145) and (596.71,109.94) .. (600,110) ;
\draw [color={rgb, 255:red, 126; green, 211; blue, 33 }  ,draw opacity=1 ][line width=2.25]  [dash pattern={on 6.75pt off 4.5pt}]  (50,110) .. controls (107,115) and (175,55) .. (290,100) .. controls (405,145) and (596.71,109.94) .. (600,110) ;

\draw (591,183) node [anchor=north west][inner sep=0.75pt]  [font=\Huge] [align=left] {$\displaystyle \mathcal{D}$};
\draw (51,22) node [anchor=north west][inner sep=0.75pt]  [font=\huge,color={rgb, 255:red, 208; green, 2; blue, 27 }  ,opacity=1 ] [align=left] {$\displaystyle \max_{\theta ,\ \phi } \ p(\mathcal{D} |\theta ,\phi ,\mathcal{M}_{X\rightarrow Y})$};
\draw (351,22) node [anchor=north west][inner sep=0.75pt]  [font=\huge,color={rgb, 255:red, 126; green, 211; blue, 33 }  ,opacity=1 ] [align=left] {$\displaystyle \max_{\theta ,\ \phi } \ p(\mathcal{D} |\theta ,\phi ,\mathcal{M}_{Y\rightarrow X})$};

\end{tikzpicture}
}

\caption{}
\end{subfigure}
\begin{subfigure}{0.33\textwidth}
\centering

\resizebox{\linewidth}{!}{

\tikzset{every picture/.style={line width=2.5pt}} %

\begin{tikzpicture}[x=0.75pt,y=0.75pt,yscale=-1,xscale=1]

\draw    (50,180) -- (272,179.61) -- (600,180) ;
\draw [color={rgb, 255:red, 208; green, 2; blue, 27 }  ,draw opacity=1 ][line width=2.25]    (100,110) .. controls (134,112) and (136,79) .. (180,80) .. controls (224,81) and (201,104) .. (250,110) ;
\draw    (100,80) -- (100,180) ;
\draw    (250,80) -- (250,180) ;
\draw    (400,80) -- (400,180) ;
\draw    (550,80) -- (550,180) ;
\draw [color={rgb, 255:red, 126; green, 211; blue, 33 }  ,draw opacity=1 ][line width=2.25]    (400,110) .. controls (434,112) and (436,79) .. (480,80) .. controls (524,81) and (501,104) .. (550,110) ;

\draw (591,183) node [anchor=north west][inner sep=0.75pt]  [font=\Huge] [align=left] {$\displaystyle \mathcal{D}$};
\draw (51,22) node [anchor=north west][inner sep=0.75pt]  [font=\huge,color={rgb, 255:red, 208; green, 2; blue, 27 }  ,opacity=1 ] [align=left] {$\displaystyle \max_{\theta ,\ \phi } \ p(\mathcal{D} |\theta ,\phi ,\mathcal{M} '_{X\rightarrow Y})$};
\draw (351,22) node [anchor=north west][inner sep=0.75pt]  [font=\huge,color={rgb, 255:red, 126; green, 211; blue, 33 }  ,opacity=1 ] [align=left] {$\displaystyle \max_{\theta ,\ \phi } \ p(\mathcal{D} |\theta ,\phi ,\mathcal{M} '_{Y\rightarrow X})$};

\end{tikzpicture}
}

\caption{}
\end{subfigure}
\begin{subfigure}{0.33\textwidth}
\centering

\resizebox{\linewidth}{!}{

\tikzset{every picture/.style={line width=2.5pt}} %

\begin{tikzpicture}[x=0.75pt,y=0.75pt,yscale=-1,xscale=1]

\draw    (50,180) -- (272,179.61) -- (600,180) ;
\draw [color={rgb, 255:red, 208; green, 2; blue, 27 }  ,draw opacity=1 ][line width=2.25]    (50,180) .. controls (73,144) and (167,24) .. (220,70) .. controls (273,116) and (324,179) .. (400,180) ;
\draw [color={rgb, 255:red, 126; green, 211; blue, 33 }  ,draw opacity=1 ][line width=2.25]    (600,180) .. controls (579,149) and (482,38) .. (430,70) .. controls (378,102) and (317,183) .. (250,180) ;

\draw (591,183) node [anchor=north west][inner sep=0.75pt]  [font=\Huge] [align=left] {$\displaystyle \mathcal{D}$};
\draw (101,12) node [anchor=north west][inner sep=0.75pt]  [font=\huge,color={rgb, 255:red, 208; green, 2; blue, 27 }  ,opacity=1 ] [align=left] {$p(\displaystyle \mathcal{D} |\mathcal{M}_{X\ \rightarrow Y})$};
\draw (371,12) node [anchor=north west][inner sep=0.75pt]  [font=\huge,color={rgb, 255:red, 126; green, 211; blue, 33 }  ,opacity=1 ] [align=left] {$p(\displaystyle \mathcal{D} |\mathcal{M}_{Y\ \rightarrow X})$};

\end{tikzpicture}
}

\caption{}
\end{subfigure}

\caption{ Toy figure with datasets on the $x$ axis and values of densities on the $y$ axis
(a) With a sufficiently flexible model, maximising the likelihood for each dataset will give the same value for both causal models in \cref{sec:freq_causal_model}. (b) This has been solved by making restrictions on the datasets they can model. (c) Bayesian model selection retains the ability to identify causal direction, while allowing flexibility. This may lead to some probability of error (overlap).}
\label{fig:max_like_dataset_density}
\end{figure*}

\section{Introduction}
Many fields use statistical models to predict the response to actions (interventions), e.g.~health outcomes after treatment. Such predictions can not be made based on correlations gained from purely observational data, but require access to causal structure \cite{pearl2009causality}. In machine learning, causal structure also helps in many prediction tasks ranging from domain adaptation \cite{wang2022unified}, robustness \cite{buhlmann2020invariance}, and generalisation \cite{scherrer2022generalization}.
Conditional independencies in data can be used to infer causal structures, but only up to a Markov equivalence class (MEC) (\cref{app:background}) \cite{pearl2009causality}.
Further assumptions are required to completely recover a causal structure.
In this paper, we aim to identify causal direction within a Markov equivalence class. 
Hence, for simplicity, we focus solely on distinguishing $\nomodelxcausey$ and $\nomodelycausex$\footnote{For clarity, we colour the cause blue throughout.}.

In the bivariate case, much of the causal literature assumes restrictions on a model, which when they match reality, guarantee identifiabilility of the causal direction.
For example, for data generated by an additive noise model (ANM), an ANM model achieves a higher likelihood in the causal direction, as the anti-causal direction violates the ANM assumptions \cite{hoyer2008nonlinear, zhang2015estimation}.
A method's overall usefulness is then empirically verified on a wider array of data generating distributions, including datasets where the rigid restrictions of models such as ANMs hamper their performance, as the models are misspecified.
Moreover, in these settings, causal identifiability is no longer guaranteed. 
The goal of our work is to investigate whether causality can be identified in models without hard restrictions, reducing misspecification, even if we lose \emph{strict} guarantees of identifiability.

We replace hard restrictions on models with softer ones encoded by Bayesian priors. 
Each causal direction is treated as a separate Bayesian model, after which causal discovery can be formulated as Bayesian model selection.
The causal models encode the \textit{independent causal mechanisms} (ICM) assumption in their respective causal directions.
We show that if the anti-causal factorisation violates the ICM assumption, Bayesian model selection can discern the causal direction.
This is true even for model classes whose flexibility prevents likelihood based discrimination.
We then analyse the probability of error of the method, both when the model is correct and when it is misspecified.
Our work shows that when faced with datasets for which no restricting assumptions can reasonably be made, our framework allows for use of an appropriate model that can better enable identification of causal direction, relative to a more restricted model with inappropriate assumptions.

To demonstrate the usefulness of our approach, we use a Gaussian process latent variable  model (GPLVM) \cite{titsias2010bayesian}, which has the ability to model a wide range of densities. 
We test this on a range of benchmark datasets with various data generating assumptions.
We also compare against previously proposed methods, both those which rely on strict restrictions, and those which are more flexible, but lack formal identifiability guarantees.
Whereas most methods perform well on the types of datasets where their assumptions hold, we find that our method performs well across all datasets.
Our findings show that causal discovery is possible without losing the ability to model datasets well, a property that is desirable in real world cases.

\section{Preliminaries \& Assumptions}

In this paper, we focus on the problem of bivariate causal discovery, under the assumption that there are no hidden confounders. 
We first introduce notation.

Throughout, given a space $\mathcal{Z}$, we write $\mathcal{P}\left(\mathcal{Z}\right)$ for the space of probability measures over $\mathcal{Z}$. 
Similarly, given a pair of spaces $\mathcal{Z},\mathcal{Z}^{\prime}$, we write $\mathcal{K}\left(\mathcal{Z},\mathcal{Z}^{\prime}\right)$ for the space of Markov kernels from $\mathcal{Z}$ to $\mathcal{Z}^{\prime}$.
In all models considered, individual bivariate observations take the form $\left(x, y\right) \in \mathcal{X} \times \mathcal{Y}$, where $\mathcal{X},\mathcal{Y}$ are (potentially distinct) spaces. 
We will view $\left(x, y\right)$ as realisations of random variables, using capital letters (e.g. $P$) to denote the associated probability measure in $\mathcal{P}\left(\mathcal{X}\times\mathcal{Y}\right)$ (which we will refer to as a `data-generating process'), and lower-case
letters (e.g. $p$) to denote the associated probability density with respect to a suitable reference measure. 
Our `datasets' will then be modelled as $N\in\mathbf{N}$ iid observations from this probability measure, which we abbreviate as $\mathcal{D}^{N} = \left(\mathbf{x}^{N}, \mathbf{y}^{N}\right) = \left\{ \left(x_{i},y_{i}\right):i\in\left[N\right]\right\} $.
Our assumption is that our data-generating process admits a causal interpretation, whereby either i) $X$ causes $Y$ ($\nomodelxcausey$), or ii) $Y$ causes $X$ ($\nomodelycausex$). 
Our goal is to then determine which of these causal directions underlies the `true' data-generating process, based on our observed dataset.

\subsection{Data Generation through Structural Causal Models}
\label{sec:data_generation}

We can express causal relationships by a Structural Causal Model (SCM), which describes the hierarchical ordering of variable generation from
causes to effects \cite{pearl2009causality}.
Each SCM can be represented as a \textit{Directed Acyclic Graph} (DAG) with a vertex for each variable.
We denote the parents of a vertex $j$ in a DAG $\mathcal{G}$ as $\text{pa}_{\mathcal{G}}(j)$.
In our bivariate setting, the causal direction $\nomodelxcausey$ corresponds to a data-generating process of the form 
\begin{align}
    \quad X_{i}:=f_{X}\left(N_{i}^{X}\right),\quad Y_{i}:=f_{Y}\left(X_{i},N_{i}^{Y}\right)
    \label{eq:SCM_eq}
\end{align}
for $i = 1, \ldots, N$, where $f_{X}$ and $f_{Y}$ are deterministic `generation functions', and $\left\{ N_{i}^{X}:i\in\left[N\right]\right\} \overset{\mathrm{iid}}{\sim}\nu^{X}$, $\left\{ N_{i}^{Y}:i\in\left[N\right]\right\} \overset{\mathrm{iid}}{\sim}\nu^{Y}$ are iid realisations of some mutually independent `noise' random variables.
This naturally induces a factorisation $P_{X,Y} \left(\mathrm{d}\left(x,y\right)\right)=P_{X}\left(\mathrm{d}x\right)\cdot P_{Y\mid X}\left(\mathrm{d}y\mid x\right)$, which we term the `causal factorisation' of this SCM. 
By disintegration of measure, one can equally factorise $P_{X,Y}\left(\mathrm{d}\left(x,y\right)\right)=P_{Y}\left(\mathrm{d}y\right)\cdot P_{X\mid Y}\left(\mathrm{d}x\mid y\right)$, which we term the `anti-causal factorisation' of the same SCM.
The SCM for $\nomodelycausex$ is analogously defined with the causal factorisation being $P_{Y}\left(\mathrm{d}y\right)\cdot P_{X\mid Y}\left(\mathrm{d}x\mid y\right)$.

\subsection{Interventions and Independent Causal Mechanisms (ICM)}
\label{sec:icm}

An intervention on a variable is an action which alters its value, generation function, or noise input, while leaving those of all other
variables unchanged. 
In the causal factorisation, an intervention on a cause means that only the distribution of the cause is changed, while the conditional distribution of effect given cause remains unchanged.
For example, if our model assumes the causal direction $\nomodelxcausey$, then an intervention on $X$ can be achieved by changing either $f_{X}$ or $N_{i}^{X}$ in \cref{eq:SCM_eq}.
Such an action changes only changes one term in the causal factorisation $P_X(\calcd x)$, but will leave $P_{Y \mid X}(\calcd y \mid x)$ invariant. 
By contrast, the same action will often induce a change in both terms of the anti-causal factorisation, $P_{Y}\left(\mathrm{d}y\right)$ and 
$P_{X\mid Y}\left(\mathrm{d}x\mid y\right)$.
Given this invariance, we thus say that the causal factorisation satisfies the Independent Causal Mechanisms (ICM) assumption \citep[ch.~2]{peters2017elements}.
The ICM assumption implies a fundamental asymmetry for the impact of interventions on different factorisations of the same joint distribution.

\subsection{Causal Models}
\label{sec:freq_causal_model}

A causal model which encodes the ICM assumption should directly specify the causal factorisation.
Specifying terms of the causal factorisation ensures that parts of the model remain well specified under interventions.
Conversely, a model which specifies the anti-causal factorisation will often find both of its components to be mismatched after an intervention \cite{scholkopf2012causal}.
We define a causal model in line with the ICM assumption.

\begin{definition}
    A \textbf{causal model} is a tuple $\mathcal{M}_{\mathcal{G}} = (\mathcal{G}, \mathcal{C}, \mathcal{F})$, where $\mathcal{G}$ is a DAG with vertex set $\mathcal{V}$, and $\mathcal{C}$ is a set of conditional distributions that specifies the causal factorisation
    \begin{align}
        \mathcal{C} = \prod_{i \in \mathcal{V}} \mathcal{C}_{i \mid \text{pa}_{\mathcal{G}}(i)} \subset \prod_{i \in \mathcal{V}} \mathcal{K}(\mathcal{X}_{\text{pa}_{\mathcal{G}}(i)} \to \mathcal{X}_i).
    \end{align}
     Given a tuple of conditionals, one for each variable, $P = (P_i: i \in \mathcal{V}) \in \mathcal{C}$, define  $\delta_{\mathcal{C}}: \mathcal{C} \to \mathcal{P}(\mathcal{X}\times \mathcal{Y})$ as the map that assembles $P$ into the corresponding joint
    \begin{align}
        \delta_{\mathcal{C}}(P)(\calcd x_{\mathcal{V}}) = \prod_{i \in \mathcal{V}} P_i(\calcd x_i \mid x_{\text{pa}_{\mathcal{G}}(i)}).
    \end{align}
    Finally, define $\mathcal{F}$ as the set of induced joint distributions $\mathcal{F} = \{ \delta_{\mathcal{C}}(P) : P \in \mathcal{C} \}$.
\end{definition}

For example, $\xcausey$ specifies $\mathcal{C} = \mathcal{C}_{X} \times \mathcal{C}_{Y|X}$, with an induced joint $\delta_{\mathcal{C}}(P_X, P_{Y|X})(\calcd x, \calcd y) = P_X(\calcd x) P_{Y|X}(\calcd y)$ for $P_X \in \mathcal{C}_{X}, P_{Y|X} \in \mathcal{C}_{Y|X}$. 
Note that the mapping $\delta_{\mathcal{C}}$ is injective. 
Throughout, we work directly with distributions for generality.
In practice, the elements of $\mathcal{C}$ are constructed as a parameterised family (though the parameter may in principle be infinite-dimensional).
When relevant, we denote the parameterisation by parameter spaces $\Phi$ and $\Theta$ for $X$ and $Y$ respectively, as in \cref{fig:ICM_graphical_model}.
To emphasise the direction of the causal model, we will often append $\xcausey$ and $\ycausex$ to distributions and densities of interest.

We are interested in the case where the causal structure is \textit{not} known in advance and we seek to determine it from data. 
In the case that both $\mathcal{F}_{\nomodelxcausey}$ and $\mathcal{F}_{\nomodelycausex}$ are sufficiently expressive such that there are few restrictions on which joint distributions on $\mathcal{X}\times\mathcal{Y}$ can be learned (in principle), then model selection based on maximum likelihood will fail to identify causal direction, as it assigns equal scores to the `best' model compatible with each causal direction \cite{zhang2015estimation}. 
The notion of distribution-equivalence \cite{geiger2002parameter} helps to capture this principle.

\begin{definition}
Two causal models $\xcausey = \left(\nomodelxcausey, \mathcal{C}_{\nomodelxcausey}, \mathcal{F}_{\nomodelxcausey}\right)$ and $\ycausex = \left(\nomodelycausex, \mathcal{C}_{\nomodelycausex}, \mathcal{F}_{\nomodelycausex}\right)$ are \textbf{distribution-equivalent} if $\mathcal{F}_{\nomodelxcausey}=\mathcal{F}_{\nomodelycausex}$. 
Equivalently, there exists a unique translating bijection $\gamma: \mathcal{C}_{\nomodelxcausey} \to \mathcal{C}_{\nomodelycausex}$ such that for any $P\in\mathcal{C}_{\nomodelxcausey}$, there holds an equality of (joint) measures $\delta_{\mathcal{C}_{\nomodelxcausey}}(P) = \delta_{\mathcal{C}_{\nomodelycausex}}( \gamma(P))$.
\label{def:distr_equiv}
\end{definition}

In short, for every $\left(m,c\right)\in\mathcal{C}_{X}\times \mathcal{C}_{Y \mid X}$, there exists $\left(m^{\prime}, c^{\prime}\right) \in\mathcal{C}_{Y} \times \mathcal{C}_{X \mid Y}$ such that $m\left(\mathrm{d}x\right)\cdot c\left(\mathrm{d}y\mid x\right)=m^{\prime}\left(\mathrm{d}y\right)\cdot c^{\prime}\left(\mathrm{d}x\mid y\right)$ as probability measures, and vice versa. 
A solution to restore identifiability is to restrict  $\mathcal{C}_{\nomodelxcausey},\mathcal{C}_{\nomodelycausex}$, but this comes at the cost of not being able to learn some distributions (\cref{fig:max_like_dataset_density} (a,b), \cref{app:likelihood_equivalence}), i.e. a loss in modelling flexibility.
Our aim in the following sections is to allow for discovering causal relations, even with distribution-equivalent models.

\section{Related Work}
\label{sec:related}

One class of methods makes hard restrictions on the set of distributions ($\mathcal{C}$) which induce the causal factorisation. 
For example, linear function with non-Gaussian noise (LiNGAM) \cite{shimizu2006linear}, non-linear functions with additive noise (ANM) \cite{hoyer2008nonlinear}, and post non-linear models (PNL) \cite{zhang2012identifiability}, among others \cite{immer2022identifiability}.
These restrictions can crudely be thought of as controlling the complexity of $\mathcal{F}$.
Identifiability is proven by showing that the more complex anti-causal factorisation lies outside of the core model.
\citet{zhang2015estimation} showed that the likelihood can be used to identify the causal direction in these models (\cref{fig:max_like_dataset_density}(b)), but if the dataset is generated by a model without these restrictions, it is possible for both causal directions to achieve similar likelihoods \cite{zhang2015estimation}. 

Another class of methods assumes more flexibility but try and control a measure of complexity.
\citet{marx2019identifiability} (SLOPPY) build on previous non-identifiable methods (RECI \cite{blobaum2018cause}, QCCD \cite{tagasovska2020distinguishing}) and  assume that the causal factorisation has been generated by a model with fewer parameters than the anti-causal factorisation.
Balancing mean squared error along with the number of parameters can then identify the causal direction.
However, measuring complexity by the number of parameters can be parametrisation-dependent.
Such conceptual problems are amplified in the setting of non-parametric or over-parametrised models.
Additionally, SLOPPY also assumes low noise and additive noise.
CGNN \cite{goudet2018learning} forego strict identifiability and try to learn a causal generative model of the data using neural networks.
Their methods works with small networks but with no clear way of mitigating overfitting, their method easily achieves the same score for both causal factorisations.

Other methods either try to i) directly measure the dependence of the factorisation, based on the ICM principle, or ii) measure the complexity of a proposed direction.
CURE \cite{sgouritsa2015inference} and IGCI \cite{daniusis2012inferring} try to measure the dependence of the factorisations.
CDCI \cite{duong2021bivariate},  CDS \cite{fonollosa2019conditional} and KCDC \cite{mitrovic2018causal} try to measure the stability of the conditional distributions under different input values, arguing that the more stable conditional is more likely the cause. 
Of the above, only IGCI has proven identifiability.

We base our approach on Bayesian model selection, which has automatic mechanisms of balancing model fit and complexity. 
The prescribed procedure is straightforward ( \cref{sec:bayesian_model_good}), and was first investigated in the 90s in the context of finding Bayesian network structure \cite{heckerman1995learning,heckerman2006bayesian,heckerman1995bayesian,chickering2002optimal, geiger2002parameter}. 
However, while we argue that Bayesian model selection is helpful for causal discovery \emph{within} a Markov Equivalence Class (MEC), these early papers restricted their focus to finding network structure up to an MEC. 
This is due to a focus on linear causal relationships, which is one key setting where Bayesian model selection does not provide much benefit (see \cref{sec:non_identifiability_gaussian} for a discussion).
Indeed, the wider benefits of Bayesian methods to infer causality within an MEC has been touched upon in \citet[ch. 35]{mackay2003information}.
\citet{friedman2000gaussian} are first to apply Bayesian model selection to do bivariate causal discovery (i.e.~within an MEC). 
They compare two Gaussian process regression models, and attempt to determine which variable should be used as the input (cause). 
However, since this model was effectively an ANM, Bayesian model selection provided little added value, since causal direction is already identifiable by the likelihood alone \cite{zhang2015estimation} ( \cref{sec:freq_causal_model}). 
A similar approach was used by \citet{kurthen2019bayesian}.
\citet{stegle2010probabilistic} highlighted these issues by noting that \citet{friedman2000gaussian} worked only due to model fit. 
Like us, they acknowledge the larger benefit when $\mathcal{C}$ is not restricted. 
While their method is similar to Bayesian model selection, it is heuristically justified by Kolmogorov complexity \citep{janzing2010justifying} (see \cref{sec:kolm_complexity_doesnt_work}); as such, the impact of both model and prior on their procedure is unclear.
This is explicit in our work (\cref{sec:asymmetry_theorem}, \cref{sec:correctnes_of_bms}).

Our contributions follow in the path of \citet{friedman2000gaussian} and \citet{stegle2010probabilistic}. 
Specifically, we provide a more complete view of why and under what conditions Bayesian model selection can identify causal direction (\cref{sec:asymmetry_theorem}).
Our work gives insight into the performance when a chosen model correctly encodes the assumptions on a dataset (\cref{sec:correctnes_of_bms}), and when it does not (\cref{sec:model_misspecification}).

\section{Bayesian Inference of Causal Direction}

We explain how causal discovery can be viewed as a Bayesian model selection problem, outlining the requisite assumptions. 
We then give conditions under which Bayesian model selection discriminates even distribution-equivalent causal models.
Correctness then depends on the exact assumptions made in the model, and how well they match reality.
We analyse the case when the assumptions are correct, providing a statistical test that can be used to quantify the probability of error inherent in the procedure.
We also provide analysis when the assumptions might be wrong.

\subsection{Bayesian Model Selection for Causal Inference from First Principles}
\label{sec:bayesian_model_good}

The maximum likelihood approach outlined in \cref{sec:freq_causal_model} was restricted in its ability to simultaneously i) estimate model parameters and ii) infer causal direction. 
The core issue is that for richly-parameterised models, both causal models can express distributions which explain the observed data equally well. 
Bayesian inference provides a general framework for inferring unknown quantities in statistical models \cite{gelman2013bayesian, mackay2003information, bernardo2009bayesian} which offers solutions to this problem. 
In this section, we will describe how Bayesian inference allows us to directly infer our quantity of interest --- the causal direction --- and
the additional assumptions which need to be specified in order for this approach to succeed.

Inferring causal direction can be seen as a \textit{Bayesian model selection} problem, where we seek to distinguish which causal model (\cref{sec:freq_causal_model}) is more likely to have generated a dataset. 
The evidence for each causal direction is quantified in the posterior
\begin{align}
    P(\xcausey|\data) = \frac{P(\data|\xcausey)P(\xcausey)}{P(\data)} \,.
\end{align}
We can summarise the balance of evidence for both causal directions with the log ratio:
\begin{align}
    \log \frac{P(\xcausey|\data)}{P(\ycausex|\data)} =
    \log \frac{P(\data|\xcausey)P(\xcausey)}{P(\data|\ycausex)P(\ycausex)} \label{eq:logodds} \,.
\end{align}
Bayesian inference requires us to specify a \emph{prior} on which causal direction is more likely. 
To represent our lack of specific knowledge, we set these prior probabilities to be equal $P(\xcausey) = P(\ycausex) = 0.5$.
Thus, the above log posterior ratio will determined only by the ratio $P(\data|\xcausey) / P(\data|\ycausex)$.

To find the ratio in \cref{eq:logodds}, we are required again to specify prior measures, this time on $\mathcal{C}$.
We denote prior over distributions by $\pi$ (as opposed to distributions over observations).
This augments a causal model to form a Bayesian Causal Model (BCM).
\begin{definition}
    A \textbf{Bayesian causal model} (BCM) is a causal model $\mathcal{M}_{\mathcal{G}}$ equipped with a prior distribution over $\mathcal{C}$, that is a tuple $(\mathcal{G}, \mathcal{C}, \mathcal{F}, \pi)$, shortened to $(\mathcal{M}_{\mathcal{G}}, \pi)$.
\end{definition}
While we will discuss the consequences of selecting a particular prior later, we can now already determine some of its properties from our problem setting. 
A strict view of the ICM assumption implies that information about the distribution on causes should not provide information on the distribution of effect given cause. 
This implies that the prior should be separable, with respect to the causal factorisation, for both causal models.\footnote{\citet{guo2022causal} provide the most direct argument for this, by proving a ``causal de Finetti'' theorem, based on the requirement that additional information on the cause mechanism does not give information on the effect mechanism. Earlier, \citet{janzing2010causal} also argue that this must be the case, but from a more heuristic argument based on Kolmogorov complexity. The assumption has also been made in earlier methods \cite{stegle2010probabilistic,sgouritsa2015inference, heckerman1995bayesian}.}.
We formally define separability of priors as follows.
\begin{definition}
    Given a BCM $(\mathcal{M}_{\mathcal{G}}, \pi)$, we call prior $\pi$ \textbf{separable} with respect to $\mathcal{C}$ if it factorises as $\prod_{i \in \mathcal{V}} \pi_{i}$ for some $(\pi_i: i \in \mathcal{V}) \in \prod_{i \in \mathcal{V}} \mathcal{P}(\mathcal{C}_{i \mid \text{pa}_{\mathcal{G}}(i)})$.
\end{definition}
For $\xcausey$, this amounts to saying that the prior factorises as $\pi_{X}(P_X)\pi_{Y}(P_{Y|X})$ with $P_X \in \mathcal{C}_{X}$ and $P_{Y|X} \in \mathcal{C}_{Y|X}$.
For parametrised models, our BCMs can be represented as the graphical models in \cref{fig:ICM_graphical_model}.

Given a prior, any BCM $\left(\mathcal{M}_{\mathcal{G}},\pi\right)$ naturally induces \textit{data distribution}, i.e. a probability measure over sequences of bivariate observations, $P\left(\cdot\mid\mathcal{M}\right)$.
This is found by integrating the joint over the prior distribution.
\begin{definition}
    The \textbf{data distribution} of  $(\mathcal{M}_{\mathcal{G}}, \pi)$ is 
    \begin{align}
        P(\calcd x, \calcd y| \mathcal{M}_{\mathcal{G}}) = \int_{P^{\prime} \in \mathcal{C}}  \delta_{\mathcal{C}}(P^{\prime} )(\calcd x, \calcd y|  \mathcal{M}_{\mathcal{G}}) \pi(\calcd P^{\prime}  | \mathcal{M}_{\mathcal{G}}).
    \nonumber
    \end{align}
    The density of this measure with respect to a suitable reference measure is known as the \textbf{marginal likelihood}.
    \label{def:marg_like}
\end{definition}

Given a dataset, and under the above procedure, the most likely model $\mathcal{M}^*$ can be chosen as the one with the higher marginal likelihood
\begin{align}
\mathcal{M}^* = \begin{cases} \xcausey \text{ if } p(\data| \xcausey) > p(\data| \ycausex) \\ \ycausex \text{ if } p(\data| \xcausey) < p(\data| \ycausex) \end{cases}.
\label{eq:decision_rule}
\end{align}
This is straightforward Bayesian model selection, but applied to models that incorporate the causal assumptions of \cref{sec:freq_causal_model}. 
Following this procedure from first principles, we see that Bayesian inference \emph{prescribes} that we must specify priors $\pi$.
In specifying these priors, one implicitly constrains the sets distributions which will be used to explain the observed data.
This is consistent with earlier work, where hard constraints on these sets are imposed to ensure identifiability \cite{zhang2015estimation}. 
While these hard constraints can be represented in the priors by only assigning non-zero mass to the desired regions, priors also allow the specification of softer constraints (\cref{fig:max_like_dataset_density}(c)).

While Bayesian inference prescribes the procedure which our method will follow, it is not clear this will produce correct answers. 
The performance of all Bayesian methods, including ours, depends on how compatible our prior assumptions are with reality.
This is also the case for previous causal discovery methods that provide identifiability guarantees, since the assumptions made have to hold in reality.
In the next sections, we investigate the influence of prior design on the performance of Bayesian model selection.
Specifically, we will analyse conditions under which choices of priors cannot imply the same data distributions for the two BCMs, the guarantees we can get when the assumptions made are correct, and when they are incorrect.

\begin{figure}[t!]
\begin{subfigure}{0.48\textwidth}
    \centering
  \tikz{
 \node[obs] (x) {$X_i$};%
 \node[obs,right=of x,xshift=0.25cm] (y) {$Y_i$};
 \node[latent,left=of x,xshift=-0.15cm] (phi) {$\phi$};
 \node[latent,right=of y,xshift=0.15cm] (theta) {$\theta$};
 \plate[inner sep=0.3cm, xshift=0cm, yshift=0.12cm] {plate1} {(x) (y)} {$i = 1,\ldots, N$}; 
 \edge {x} {y};
 \edge {phi} {x};
 \edge {theta} {y} }
 \caption{}
\end{subfigure}
\begin{subfigure}{0.48\textwidth}
    \centering
  \tikz{
 \node[obs] (x) {$X_i$};%
 \node[obs,right=of x,xshift=0.25cm] (y) {$Y_i$};
 \node[latent,left=of x,xshift=-0.15cm] (phi) {$\phi$};
 \node[latent,right=of y,xshift=0.15cm] (theta) {$\theta$};
 \plate[inner sep=0.3cm, xshift=0cm, yshift=0.12cm] {plate1} {(x) (y)} {$i = 1,\ldots, N$}; 
 \edge {y} {x};
 \edge {phi} {x};
 \edge {theta} {y}  }
 \caption{}
\end{subfigure}
\caption{Graphical models for parametrised Bayesian causal models $\xcausey$ and $\ycausex$. The causal direction indicates the factorisation that encodes ICM.}
\label{fig:ICM_graphical_model}
\end{figure}

\subsection{Asymmetry Between Dataset Densities of Bayesian Causal Models}
\label{sec:asymmetry_theorem}

In \cref{sec:freq_causal_model} we discussed that the maximum likelihood score is indifferent to causal direction when the causal models are distribution-equivalent.
In general, we can expect this to cause difficulties when the distributions specified ($\mathcal{C}$) are sufficiently flexible.
By contrast, Bayesian inferences prescribes using the \textit{marginal likelihood} to guide model selection. 
Here, we show that Bayesian model selection can be sensitive to causal direction, even when using distribution-equivalent causal models. 
This shows that the marginal likelihood is capable of providing a preference between factorisations in situations where maximum likelihood cannot.
While this result does not, by itself, imply that Bayesian model selection will identify the correct causal direction (which we discuss later),
it does offer insight into the difference between the two approaches.
All proofs are in \cref{app:sec:proofs}.

We will find a necessary condition under which marginal likelihood cannot discriminate BCMs.
If the condition does not hold, the BCMs can surely be discerned.
We thus define a notion of equivalence, similar to \cref{def:distr_equiv}, but for BCMs and marginal likelihoods. 
\begin{definition}
Given two BCMs $(\xcausey, \pi_{\nomodelxcausey})$,
$(\ycausex, \pi_{\nomodelycausex})$, say they are \textbf{Bayesian distribution-equivalent} if   $P\left(\cdot\mid\xcausey\right)=P\left(\cdot\mid\ycausex\right)$,
i.e. for all $N\in\mathbf{N}$, and for all $\left(\mathbf{x}^{N},\mathbf{y}^{N}\right)\in\left(\mathcal{X}\times\mathcal{Y}\right)^{N}$,
it holds that $p\left(\mathbf{x}^{N},\mathbf{y}^{N}\mid\xcausey\right)=p\left(\mathbf{x}^{N},\mathbf{y}^{N}\mid\ycausex\right)$.
\end{definition}

In the following, we show that if two causal models are not distribution-equivalent, then there exist Bayesian causal models that are not Bayesian-distribution equivalent.
This demonstrates that if maximum likelihood can distinguish causal models, then suitably constructed Bayesian causal models can be differentiated using the marginal likelihood.
\begin{proposition}
Given two BCMs $(\xcausey, \pi_{\nomodelxcausey})$, $(\ycausex, \pi_{\nomodelycausex})$, suppose that there exists a subset $\mathcal{C}_{\Delta} \subset \mathcal{C}_{\nomodelxcausey}$ such that $\pi_{\nomodelxcausey}(\mathcal{C}_{\Delta}) > 0$, and $\delta_{\mathcal{C}_{\nomodelxcausey}}(\mathcal{C}_{\Delta}) \cap \mathcal{F}_{\nomodelycausex}$ is empty.
Then the two Bayesian causal models are not Bayesian distribution-equivalent.
\end{proposition}

We are ultimately interested in the case where the underlying causal models are distribution-equivalent.
To help with this, we introduce \emph{separable-compatibility}.
\begin{definition}
Let $(\xcausey, \pi_{\nomodelxcausey})$,
$(\ycausex, \pi_{\nomodelycausex})$ be two BCMs where the underlying causal models are distribution-equivalent, denoting $\gamma$ as the corresponding translation mapping $\gamma: \mathcal{C}_{\nomodelxcausey} \to \mathcal{C}_{\nomodelycausex}$ (in \cref{def:distr_equiv}). Say the two are \textbf{separable-compatible} if: i) the pushforward $\gamma\#\pi_{\nomodelxcausey}$ is separable with respect to $\mathcal{C}_{\nomodelycausex}$, ii) $\gamma^{-1}\#\pi_{\nomodelycausex}$ is separable with respect to $\mathcal{C}_{\nomodelxcausey}$.
\end{definition}

Under the above definition, we see that the ICM assumption is preserved in the anti-causal factorization.
We now show that if two Bayesian causal models are not separable-compatible, they cannot be Bayesian distribution-equivalent.

\begin{proposition}
Given two BCMs $(\xcausey, \pi_{\nomodelxcausey})$, $(\ycausex, \pi_{\nomodelycausex})$, where the underlying causal models are distribution-equivalent, the two Bayesian causal models are Bayesian distribution-equivalent only if they are separable-compatible.
\end{proposition}

For parametrised models, we require that the parametrisation is injective for the same result to hold.
Being separable-compatible then implies that each BCM can factorise as both graphical models in \cref{fig:ICM_graphical_model}.
\begin{corollary}
Given two injectively parametrised BCMs, $\mathcal{B}^1 := (\xcausey, \pi_{\nomodelxcausey})$ that factorises as \cref{fig:ICM_graphical_model}(a), and $\mathcal{B}^2 := (\ycausex, \pi_{\nomodelycausex})$ that factorises as \cref{fig:ICM_graphical_model}(b), assume the underlying causal models are distribution-equivalent. The two Bayesian causal models are Bayesian distribution-equivalent only if: i) $\mathcal{B}^1$ also factorises as \cref{fig:ICM_graphical_model} (b), ii) $\mathcal{B}^2$ also factorises as \cref{fig:ICM_graphical_model} (a).
\label{cor:parametrise_corrollary}
\end{corollary}

This result can be interpreted as follows: given models corresponding to each causal direction, suppose that both models are distribution-equivalent. 
Supposing that the priors for each causal direction are specified in line with the ICM assumption, one can compute the anti-causal factorisation of each model, and see which prior it induces in the converse model. 
If this induced prior is incompatible with the ICM assumption (i.e. it is not separable with respect to the anti-causal factorisation), then we can deduce that the two models cannot be Bayesian distribution-equivalent.
In short, if ICM does not hold in the anti-causal factorisation, there exist datasets for which the marginal likelihood of the two models is not equal.
This is a necessary condition for distinguishing causal directions.
We analyse specific models in \cref{sec:analysis_of_models}.
Next, we analyse correctness and provide a  test to quantify the level overlap between the models.

\begin{figure*}[t]
    \centering
    \begin{subfigure}{0.48\textwidth}
    \resizebox{\linewidth}{!}{
        \includegraphics[width=\textwidth]{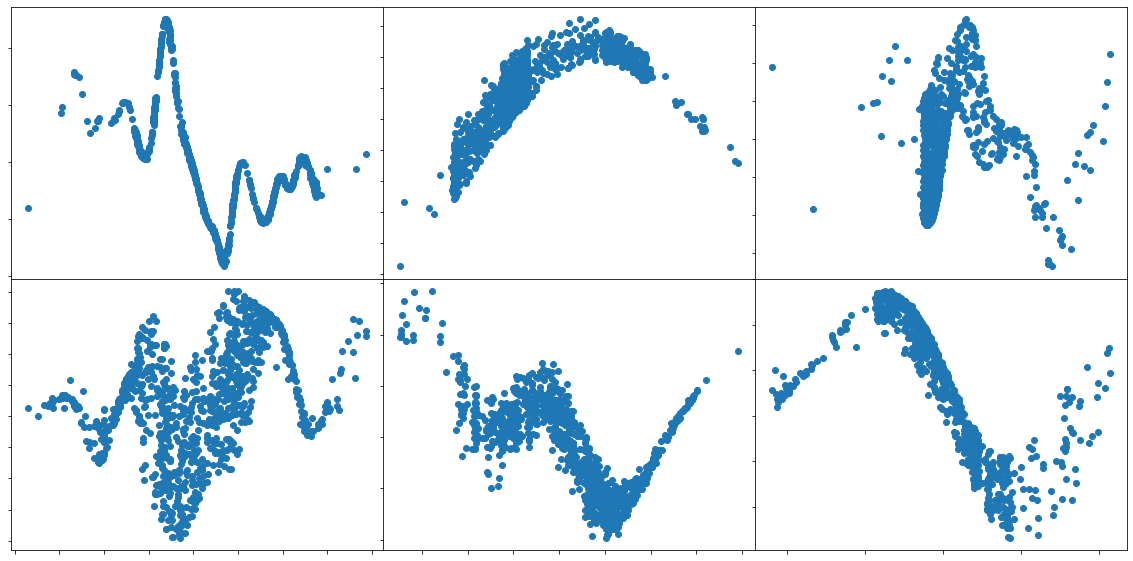}
    }
    \caption{}
    \end{subfigure}
    \begin{subfigure}{0.48\textwidth}
    \resizebox{\linewidth}{!}{
        \includegraphics[width=\textwidth]{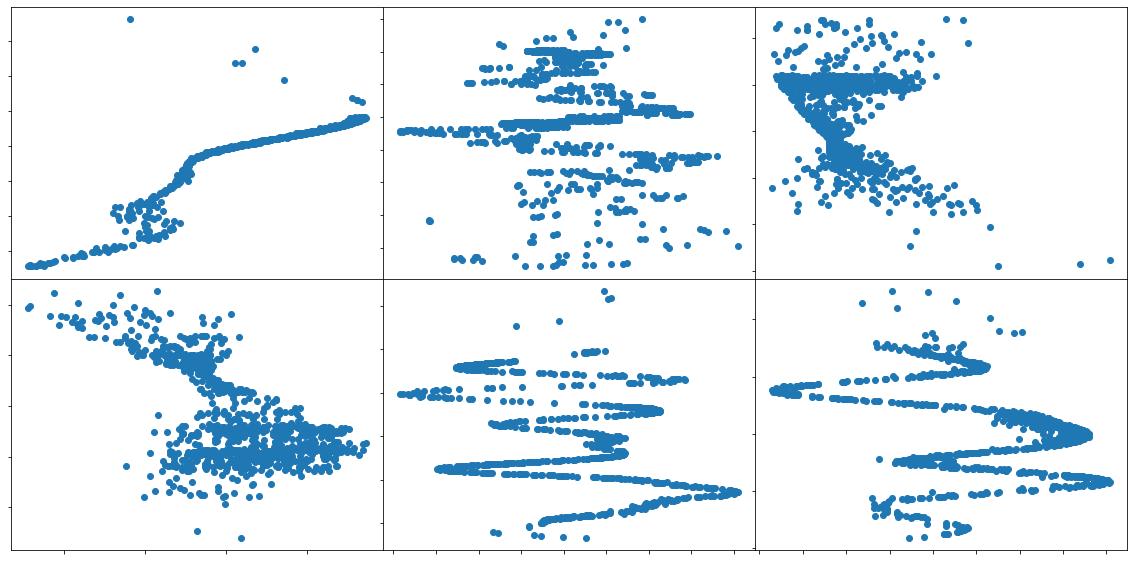}
    }
    \caption{}
    \end{subfigure}
    \caption{Samples of datasets from our chosen GPLVM model. ALL figures have the variable $X$ on the x-axis and the variable $Y$ on the y-axis. (a) Shows 6 datasets sampled from GPLVM with $\xcausey$. (b) Shows 6 datasets sampled from GPLVM with $\ycausex$. The figures show that the data distribution varies between the two Bayesian causal models.}
    \label{fig:model_samples}
\end{figure*}

\subsection{Correctness of Bayesian Model Selection}
\label{sec:correctnes_of_bms}

Distinguishing between models is necessary but not sufficient for correctly identifying the causal direction. 
Correctness depends on how well the assumptions made in the method match reality. 
We follow the causal literature by asking when Bayesian model selection will identify the correct causal direction, when the model is correctly specified.
The causal literature commonly aims to prove a strict notion of identifiability, where the correct causal direction is \emph{always} be recovered as $N\!\to\!\infty$ when assumptions hold. 
Here, we follow \citet{guyon2019evaluation} in quantifying the probability of making an error in identifying the causal direction.

The assumptions made in a Bayesian model are specified by the model structure and prior densities. 
In this section, we thus assume that our assumptions hold, and that the data is generated from the chosen BCMs.
Given the true causal direction and the decision rule in \cref{eq:decision_rule}, we find the probability of error by integrating over the region where the wrong model would be selected, i.e.
\begin{align}
    P(E| \xcausey) =& \int_{\mathcal{R}_Y} p(\data| \xcausey ) \calcd \data \,, \label{eq:overlap_error}
\end{align}
where $\mathcal{R}_Y = \{\data \suchthat p(\data|\ycausex) > p(\data|\xcausey)\}$.
The total probability of error can be written as  (\cref{sec:app:total_prob_error})
\begin{align}
    P(E) = \frac{1}{2} (1-  \text{TV}[P(\cdot| \xcausey ),  P(\cdot| \ycausex )]), 
    \label{eq:p_error}
\end{align}
where $\text{TV}(P,Q) \in [0,1]$ is the total variation distance.
We can see that the probability of making an error falls as the distance between the data distributions increases.
For identifiable models where the data distributions are completely separate, $\text{TV} = 1$ and hence $P(E) = 0$, showing that the same guarantees hold for Bayesian model selection as in these strictly identifiable settings.
For models without hard restrictions, this probability of error may be positive (overlap in \cref{fig:max_like_dataset_density}(c)). 
However, this may still be preferable to the error incurred by using an overly-restricted model.

\paragraph{Statistical Test of Asymmetry:}
We can estimate \cref{eq:p_error} to find the error under the assumption that a model is correct.
This can be used as a statistical test for whether causality can be discerned between two BCMs.
The procedure is to sample multiple datasets from the two BCMs, and classify the causal direction by using \cref{eq:decision_rule}.
We do this here for our model choice: the GPLVM\footnote{We use an approximation to the marginal likelihood that we describe later.}.
For 100 datasets of 1000 samples each, we get an approximate probability of error of 0 with a standard deviation upper bounded by $0.05$ (\cref{sec:derivation_poe}).
To aid intuition, we plot samples (\cref{fig:model_samples}) from the GPLVM prior for the two causal assumptions.
Here, we can see that while datasets from $\xcausey$ are multimodal when conditioned on $Y$, datasets from $\ycausex$ are multimodal when conditioned on $X$.
This shows that the likely datasets are strongly influenced by the causal assumptions.
When the GPLVM is a good description of the data generating process, we thus expect to obtain the correct causal direction with high probability.

\subsection{Model Misspecification}
\label{sec:model_misspecification}

The previous result relied on assuming that our model describes the data generating process well.
When our models deviate from the true data generating process, our estimate for the probability of error is incorrect.
In this case, we can bound the difference between the true probability of error and the one of our models.
We denote the true data distributions as $Q(\cdot | \nomodelxcausey)$ and $Q(\cdot | \nomodelycausex)$, but we use our models to decide on the causal direction, i.e. the decision rule in \cref{eq:decision_rule}.
We can then find a bound (\cref{sec:app:prob_error_misspecify})
\begin{align}
    2|Q(\text{E}) \! -\! P(\text{E})| \leq  &\text{TV}[  Q(\cdot | \nomodelxcausey), P(\cdot | \xcausey)] + \nonumber\\ &\text{TV}[  Q(\cdot | \nomodelycausex), P(\cdot | \ycausex)].
    \nonumber 
\end{align}
Thus, the difference between the model error and the true error is bounded by the total variation between the true and our model's data distribution, i.e. when assumptions are violated only mildly, we can still accurately estimate the probability of error.
However, large deviations from the assumptions will result in an inaccurate probability of error.
We test the GPLVM under different data generating assumptions in \cref{sec:exp_real_synthetic}.

\begin{table*}[t]
\centering
\caption{Performance comparisons. Numbers convey the ROC AUC metric (higher is better). Best results are in bold while the best results from the baselines are underlined. Our method (GPLVM) outperforms competing methods. CDCI contains multiple methods, we show the result of the best method on each dataset.}
\begin{tabular}{llllll}
 \toprule
 Methods &   CE-Cha & CE-Multi & CE-Net & CE-Gauss & CE-Tueb\\
\midrule
  LiNGAM  & 57.8 & 62.3 & 3.3 & 72.2 & 31.1 \\
  ANM   & 43.7 & 25.5 & 87.8 & 90.7 & 63.9  \\
  PNL  & \underline{78.6} & 51.7 & 75.6 & 84.7 & 73.8 \\
  IGCI  & 55.6 & 77.8 & 57.4 & 16.0 & 63.1  \\
  RECI  & 59.0 & 94.7 & 66.0 & 71.0  & 70.5 \\
  SLOPPY  & 60.1 & 95.7 & 79.3 & 71.4 & 65.3 \\
  CGNN   & 76.2 & 94.7 & 86.3 & 89.3 & \underline{76.6} \\
  GPI   & 71.5 & 73.8 & 88.1 & 90.2 & 70.6  \\
  CDCI (best method reported) & 72.2 & \underline{96.0} &  \underline{94.3} & \textbf{\underline{91.8}} & 58.7 \\
  \midrule
   \textbf{GPLVM} & \textbf{81.9} & \textbf{97.7} & \textbf{98.9} & 89.3 & \textbf{78.3}\\
 \bottomrule
\end{tabular}
\label{table:perfect_results}
\end{table*}

\section{Model Choice}
\label{sec:method}

The practical advantage of our approach is the ability to specify causal models with a large $\mathcal{C}$, which may help with reducing model misspecification. 
To use this advantage, we choose to use the (conditional) GPLVM \cite{dutordoir2018gaussian} as a prior. 
Densities for this model are constructed by warping a Gaussian random (latent) variable by a non-linear function (as in a VAE \cite{kingma2014vae}), with a Gaussian Process (GP) prior. 
The latent and flexible prior on $f,g$ allows highly non-Gaussian and heteroskedastic distributions for $\mathcal{C}$.
For $\xcausey$, $\mathcal{C}_{\nomodelxcausey}$ takes the form
\begin{align}
    p(y_i|x_i, f, \xcausey) &= \int p(y_i|f(x_i, w_i), x_i, w_i) p(w_i) \calcd w_i \,, \nonumber\\
    p(x_i|g, \xcausey) &= \int p(x_i|g(v_i), v_i) p(v_i) \calcd v_i \,,\nonumber
\end{align}
where $w_i, v_i$ are standard Gaussian distributed and priors $f|\boldsymbol\lambda\sim\mathcal{GP}$, $g|\boldsymbol\lambda\sim\mathcal{GP}$, with $\boldsymbol{\lambda}$ collecting all hyperparameters. 
The marginal likelihoods are found by integrating over the priors of $f,g$.
We use existing variational inference schemes \cite{dutordoir2018gaussian,lalchand2022generalised} to approximate the marginal likelihoods ( \cref{app:gplvm_model_details}).
This can cause additional loss in performance compared to ideal Bayesian model comparison \cite{blei2017variational}.
However, we found our model to have a near zero probability of error with variational inference (\cref{sec:correctnes_of_bms}).
Due to the symmetry of the problem, we assume the same prior on the cause, and effect given cause in the two BCMs. 
We expand on integrating out hyperparameters in \cref{sec:evidence_approx}.

\section{Experiments}

Having laid out our method, we now test it on a mixture of real and synthetic datasets.
We test our method on benchmark datasets of a wide variety of dataset-generating distributions, showcasing the benefits of removing model restrictions.
These datasets are not sampled from our model, and hence provide empirical verification of \cref{sec:model_misspecification}, i.e. that good performance can be obtained with imperfect models.
The details of our method are in  \cref{sec:gplvm_details}. 

Following previous work \cite{mooij2016distinguishing}, we use the \textit{Area under the curve} (AUC) of the \textit{Receiver Characteristic Operator} (ROC) metric, ensuring there is $50\%$ of each causal direction present to avoid directional bias.
We normalise all datasets following \citet{reisach2021beware}.
Additional experiments are in \cref{sec:additional_exp}.

\subsection{Real and Synthetic Data}
\label{sec:exp_real_synthetic}

We test the GPLVM under model misspecification on a wide variety of data-generating mechanisms, not all generated from known identifiable models, the full details of which are in \cref{sec:dataset_details}.
The results of our method (GPLVM) are shown in \cref{table:perfect_results}, along with competing methods.
Results for SLOPPY were obtained by rerunning the author's code (\cref{sec:sloppy_details}), results for CDCI are from \cite{duong2021bivariate}, and the rest are taken from \cite{guyon2019evaluation}.
 
Patterns to note here are that methods with restrictive assumptions, in exchange for strict identifiability, do not perform very well.
Methods with weaker assumptions and without strict identifiability perform better.
The poor performance of LiNGAM \cite{shimizu2006linear} can be attributed to the fact that few of the datasets contain linear functions.
ANM is seen to perform well when the datasets contain additive noise (CE-Gauss). 
PNL is less restrictive than ANM, which contributes to its better performance on most datasets.
Methods dependent on low noise, such as IGCI, RECI, and SLOPPY only perform well on CE-Multi.
RECI and SLOPPY also rely on the additive noise assumption, explaining their similar performance; we observe that SLOPPY performs better in most cases due to its better complexity control.
More flexible methods based informally on the ICM assumption such as CGNN, GPI, and CDCI tend to perform better across all datasets.
Although CGNN uses neural networks, it requires additional datasets to tune its complexity.
GPI uses a similar model as us, but their inference method differs. 
CDCI is a class of methods and the reported results are the best of 5 different methods.
Our approach, labelled GPLVM, performs well on datasets regardless of the data generating assumptions, owing to its ability to model flexible densities.
These results demonstrate the strength of our approach in identifying causal direction for more realistic assumptions.

\section{Conclusion}

In this work, we show that causal discovery with Bayesian model selection allows for removing restrictions on modelling capability that may hamper performance on real world datasets.
Starting from first principles, we show how to view causal discovery as Bayesian model selection, encoding important causal assumptions.
We then show that Bayesian model selection can infer causality in cases where flexible model choices inhibit likelihood based discovery.
We exhibit the reversibility of the ICM assumption as the key underlying mechanism for this.
We also discuss the correctness of our method and provide a statistical method for quantifying the probability of error of chosen Bayesian priors.
We show that under mild model misspecification, the estimated probability of error can retain its accuracy.
We significantly outperform previous methods on a wide range of data generating processes, owing to the removal of restrictions on the model.
Such an approach is vital for expanding the use of causal discovery to real world datasets.
While we provide a statistical test for quantifying cross-model overlap, an open question is to find universal, practical conditions for model equivalence.
An avenue of further interest would be to use deeper models \cite{damianou2013deep}.

\section*{Impact Statement}

This paper presents work whose goal is to advance the field of Machine Learning. There are many potential societal consequences of our work, none which we feel must be specifically highlighted here.

\bibliography{bibliography}
\bibliographystyle{plainnat}

\newpage
\appendix
\onecolumn

\section{Background}
\label{app:background}

\subsection{Markov Equivalence Class}

\begin{definition}
 Given a directed acyclic graph (DAG) $\mathcal{G}$ and a joint distribution $P(X_1, \ldots, X_D)$, the distribution is said to satisfy the \textbf{Markov property} if for disjoint sets $\mathbf{X}_i, \mathbf{X}_j, \mathbf{X}_k$
\begin{align}
    \mathbf{X}_i \indep_{\mathcal{G}} \mathbf{X}_j | \mathbf{X}_k \implies \mathbf{X}_i \indep \mathbf{X}_j | \mathbf{X}_k,
\end{align}
where $\indep_{\mathcal{G}}$ denotes d-separation in $\mathcal{G}$ and $\indep$ denotes independence \cite{peters2017elements}.
\end{definition}

There can be more than one distribution that is Markov with respect to a DAG $\mathcal{G}$. If there exists a set of distributions that is Markov with respect to two distinct DAGs, $\mathcal{G}_1$ and $\mathcal{G}_2$, then $\mathcal{G}_1$ and $\mathcal{G}_2$ are \textit{Markov equivalent}.
The set of all Markov equivalent DAGs is called a \textit{Markov equivalence class}.
\citet{verma1990equivalence} provide a graphical criteria for determining whether two graphs are Markov equivalent.
Simply stated, two DAGs are Markov equivalent if they share the same adjacencies and the same colliders. 
This effectively states that the respective Markov distributions must obey the same set of conditional independences.

\subsection{Distribution-equivalence and causal identification}
\label{app:likelihood_equivalence}
In this section, we expand on the relation between distribution-equivalence and causal identifiability.
For readability, we state the relevant definitions again.

\subsubsection{Distribution-equivalence}

For clarity, we reintroduce the definition of a causal model.
\begin{definition}
    A causal model is a tuple $\mathcal{M}_{\mathcal{G}} = (\mathcal{G}, \mathcal{C}, \mathcal{F})$, where $\mathcal{G}$ is a DAG with vertex set $\mathcal{V}$, and $\mathcal{C}$ is a set of conditional distributions
    \begin{align}
        \mathcal{C} = \prod_{i \in \mathcal{V}} \mathcal{C}_{i \mid \text{pa}_{\mathcal{G}}(i)} \subset \prod_{i \in \mathcal{V}} \mathcal{K}(\mathcal{X}_{\text{pa}_{\mathcal{G}}(i)} \to \mathcal{X}_i) 
    \end{align}
    that are Markov with respect to $\mathcal{G}$. Given $P = (P_i: i \in \mathcal{V}) \in \mathcal{C}$, define  $\delta_{\mathcal{C}}: \mathcal{C} \to \mathcal{P}(\mathcal{X}\times \mathcal{Y})$ as the map that assembles $P$ into the corresponding joint
    \begin{align}
        \delta_{\mathcal{C}}(P)(\calcd x_{\mathcal{V}}) = \prod_{i \in \mathcal{V}} P_i(\calcd x_i \mid x_{\text{pa}_{\mathcal{G}(i)}}).
    \end{align}
    Finally, define $\mathcal{F}$ as the set of induced joint distributions $\mathcal{F} = \{ \delta_{\mathcal{C}}(P) : P \in \mathcal{C} \}$.
\end{definition}

We restate the definition of  distribution-equivalence from \cite{geiger2002parameter}.
\begin{definition}
Two causal models $\xcausey=\left(\nomodelxcausey, \mathcal{C}_{\nomodelxcausey}, \mathcal{F}_{\nomodelxcausey}\right)$ and $\ycausex=\left(\nomodelycausex, \mathcal{C}_{\nomodelycausex}, \mathcal{F}_{\nomodelycausex}\right)$ are \textbf{distribution-equivalent} if $\mathcal{F}_{\nomodelxcausey}=\mathcal{F}_{\nomodelycausex}$. 
Equivalently, there exists a unique translating bijection $\gamma: \mathcal{C}_{\nomodelxcausey} \to \mathcal{C}_{\nomodelycausex}$ such that for any $P\in\mathcal{C}_{\nomodelxcausey}$, there holds an equality of (joint) measures $\delta_{\mathcal{C}_{\nomodelxcausey}}(P) = \delta_{\mathcal{C}_{\nomodelycausex}}( \gamma(P))$.
\label{def:app:distr_equiv}
\end{definition}

Distribution equivalence implies that for every  $\left( m , c \right)\in\mathcal{C}_{X}\times\mathcal{C}_{Y\mid X}$, there exists $\left( m^{\prime}, c^{\prime} \right)\in\mathcal{C}_{Y}\times\mathcal{C}_{X\mid Y}$ such that
\[
m\left(\mathrm{d}x\right)\cdot c \left(\mathrm{d}y\mid x\right)=m^{\prime}\left(\mathrm{d}y\right)\cdot c^{\prime}\left(\mathrm{d}x\mid y\right).
\]
Thus, given a dataset $\mathcal{D}^{N}=\left(\mathbf{x}^{N},\mathbf{y}^{N}\right)$ (and for sufficiently large $N$), maximum likelihood cannot distinguish distribution-equivalent causal models.
Hence, causal models that are distribution-equivalent, are not identifiable by maximum likelihood.
Note that Markov equivalent graphs need not imply distribution-equivalent causal models \cite{geiger2002parameter}. 

\subsubsection{Causal identifiability.}
If two causal models, $\xcausey$ and $\ycausex$, are not distribution equivalent, then there exists a $\left( m , c \right) \in\mathcal{C}_{X} \times \mathcal{C}_{Y\mid X}$ such that for all $\left( m^{\prime}, c^{\prime} \right)\in\mathcal{C}_{Y} \times \mathcal{C}_{X\mid Y}$,
\[
m\left(\mathrm{d}x\right)\cdot c \left(\mathrm{d}y\mid x\right)\neq m^{\prime}\left(\mathrm{d}y\right)\cdot c^{\prime}\left(\mathrm{d}x\mid y\right).
\]
In this case, there exists some dataset $\mathcal{D}^N$ such that (for large enough $N$) maximum likelihood can be used to identify the causal model.

We retain the traditional notion of causal identifiability if we force $\mathcal{F}_{\nomodelxcausey}$ and $\mathcal{F}_{\nomodelycausex}$ to be disjoint.
\begin{definition} (adapted from \citet[ch.~2]{guyon2019evaluation})
    Given two causal models $\xcausey$ and $\ycausex$, the models are said to be \textbf{identifiable} if for every $\left( m , c \right) \in \mathcal{C}_{X}\times\mathcal{C}_{Y\mid X}$,  there is no $\left( m^{\prime}, c^{\prime} \right) \in \mathcal{C}_{Y}\times\mathcal{C}_{X\mid Y}$ such that
    \[m\left(\mathrm{d}x\right)\cdot c\left(\mathrm{d}y\mid x\right) = m^{\prime}\left(\mathrm{d}y\right)\cdot c^{\prime}\left(\mathrm{d}x\mid y\right).
    \]
\end{definition}
The above definition ensures that the two causal models will never get the same maximum likelihood score, regardless of the dataset.

An example of an identifiable causal model is when $\mathcal{C}$ is restricted to an additive noise form \cite{hoyer2008nonlinear}; these are collectively called additive noise models (ANM). 
In this case, we can define the conditional family that induces $\mathcal{F}_{\nomodelxcausey}$ as 
\[
\quad\mathcal{C}_{Y\mid X}=\left\{ P_{\theta_{Y\mid X}} : P_{\theta_{Y \mid X}}(Y|X)  = \theta_{Y \mid X}(X) + N,  \theta_{Y \mid X} \in \Theta_{Y \mid X}  \right\},
\]
where $\Theta_{Y \mid X}$ is an arbitrary space of non-linear functions, $N$ is sampled from some arbitrary distribution and $N \indep Y$.
For this choice of conditional distribution family, \citet[thm. 1]{hoyer2008nonlinear} show that, unless some complicated conditions hold, the backward factorisation of an ANM causal model is not expressible as an ANM model.
Simply put, suppose $\mathcal{C}_{Y\mid X}$ and $\mathcal{C}_{X\mid Y}$ are additive non-linear models, then if $P \in \mathcal{F}_{\nomodelycausex}$, then generally $P \notin \mathcal{F}_{\nomodelxcausey}$.  
Hence, additive noise models are identifiable.

\citet{zhang2015estimation} demonstrated that for a Structural Causal Model (SCM) assuming independence of noise on the effect from the cause, certain restrictions on the SCM are necessary to ensure identifiability.
If the restrictions imply that the backward factorisation does not admit independence between the noise and the input, the causal model is identifiable.
This assumption on the conditional family for $\mathcal{F}_{\nomodelxcausey}$ can be written as 
\[
\quad\mathcal{C}_{Y\mid X}=\left\{ P_{\theta_{Y\mid X}} : P_{\theta_{Y \mid X}}(Y|X)  = \theta_{Y \mid X}(X, N), Y\indep N,  \theta_{Y \mid X} \in \Theta_{Y \mid X}  \right\},
\]
with an analogous family $\mathcal{C}_{\nomodelycausex}$, with the same parameter space, $\Theta:= \Theta_{Y \mid X} =  \Theta_{X \mid Y}$ and the same noise distribution.
The causal models with this family are then identifiable if $\Theta$ is restricted appropriately such that  $P \in \mathcal{F}_{\nomodelxcausey} \implies P\notin \mathcal{F}_{\nomodelycausex}$.
\citet{zhang2015estimation} also formally show that maximum likelihood can be used to identify such models.

\section{Discussion on  Kolmogorov Complexity, Causality and Bayes}
\label{sec:kolm_complexity_doesnt_work}

\subsection{Kolmogorov Complexity and Causality}
The \textbf{Kolmogorov complexity} of a string $x$, denoted $K(x)$, is the length of the shortest computer program that prints $x$ and halts \cite{grunwald2004shannon}.
This computer program can be written in any universal language, the complexity will change based on the universal language by a constant factor not depending on $x$.
We can equally define a conditional version of Kolmogorov complexity, given an input string $y$ as the shortest program that generates $x$ from $y$ and halts --- $K(x|y)$.  
We can then think of the Kolmogorov complexity of a function, for a given input $x$, as the shortest program that generates the output $f(x)$ up to a certain precision.
The definition of the Kolmogorov complexity of a probability distribution follows.

\citet{janzing2010causal} propose using the Kolmogorov complexity of factorisations of the joint to infer causality. 
Given a causal graph $\nomodelxcausey$, they formalised the assumption of \textit{Independent Causal Mechanisms} (ICM) in terms of Kolmogorov complexity by stating that the \textit{algorithmic mutual information} of the causal factorisation is zero,
\begin{align}
    I(P_X : P_{Y|X} ) &\stackrel{+}{=} 0 \\
    \implies K(P_X, P_{Y|X}) &\stackrel{+}{=} K(P_X )+ K(P_{Y|X})     \label{eq:algorithmic_markov_1} \\
    \implies K(P_X) + K(P_{Y|X}) &\stackrel{+}{\leq} K(P_Y) + K( P_{X|Y}) \label{eq:algorithmic_markov_2}\,,
\end{align}
where $I(\cdot: \cdot)$ is the algorithmic mutual information \cite{grunwald2004shannon}, \cref{eq:algorithmic_markov_1} follows by definition and \cref{eq:algorithmic_markov_2} follows from the fact that the Kolmogorov complexity of the anti-causal factorisation cannot be less than that of the joint.
The addition symbol above the inequality relations symbolises the fact they only hold up to an additive constant.
\Cref{eq:algorithmic_markov_2} suggests that causality can be inferred by finding the factorisation with the lowest Kolmogorov complexity.
However, in addition to the fact that \cref{eq:algorithmic_markov_2} requires access to the actual distributions, and that the relation only holds up to unknown additive constants, the Kolmogorov complexity is also uncomputable \cite{grunwald2004shannon}. 
\Cref{eq:algorithmic_markov_2} has thus been used informally to try and infer causality from data \cite{goudet2018learning, mitrovic2018causal, duong2021bivariate, stegle2010probabilistic}.
These methods only use \cref{eq:algorithmic_markov_2} as a philosophical foundation, and \cref{eq:algorithmic_markov_2} does not necessarily provide guarantees that their method will return the correct causal direction.

\subsection{Minimum Description Length relaxations}

Recently there have been attempts to find an analogous inequality to \cref{eq:algorithmic_markov_2} using the \textit{Minimum Description length} (MDL)  principle \cite{grunwald2007minimum}.
First, MDL allows for reasoning about finite data that has to be used to estimate the relevant probability distributions.
\citet{marx2022formally} use relations between Shannon entropy and Kolmogorov complexity to find a formulation in terms of the Kolmogorov complexity of the model and data given the probability distribution \cite{marx2022formally}
\begin{align}
    K_{\nomodelxcausey} := K(P_X) + K(x|P_X) + K(P_{Y|X}) + K(y|x, P_{Y|X}).
    \label{eq:finite_kolm_complex}
\end{align}
In expectation, above equation will equal the left hand side of the inequality \cref{eq:algorithmic_markov_2}, 
\begin{align}
    \mathbb{E}_{P(x,y)}[K_{\nomodelxcausey}] = K(P_X) + K(P_{Y|X}).
\end{align}
Hence, the inequality in \cref{eq:algorithmic_markov_2} only holds in expectation for finite data.
Second, MDL restricts the definition of Kolmogorov complexity from the set of all programs to only those that can be computed, usually specified by a model class.
\citet{marx2022formally} further make the assumption of a model class, $\model$, and assume that the data is generated from that model class.
In this case \cref{eq:finite_kolm_complex} can be written as
\begin{align}
    L_{\xcausey} := L(\model_X) + L(\bfx| \model_X) + L(\model_{Y|X}) + L(\bfy| \bfx, \model_{Y|X}),
\end{align}
where $L$ is some encoding scheme.
As such, the above approach can be considered as balancing the fit of the model (encoding of data given the model) and the complexity of the model class (encoding of the model).
However, here the exact performance will depend on the encoding scheme used.

The Bayesian approach we have considered can be seen as a variant of the MDL principle.
Here, the data given the model and model are note encoded separately.
Specifically, due to the fact that the marginal likelihood has to normalise over datasets, it has an in built complexity penalty. 
It thus also balances model fit along with a complexity penalty.
To see this clearly, consider a model $\model$ with parameter $\rho$ and prior $p(\rho| \model)$, we can write the marginal likelihood as
\begin{align}
    p(\bfx| \model) = \mathbb{E}_{p(\rho|\bfx,\model)} [p(\bfx| \rho, \model)]- \text{KL}[p(\rho|\bfx, \model) \| p(\rho| \model)],
\end{align}
where the first term is the model fit (expectation of the likelihood under the posterior), and the second term is the complexity penalty (distance from the posterior to the prior).
Our approach can thus also be justified by using MDL arguments, though the MDL view does not provide insight into why the Bayesian approach works, nor into the consequences of the choice of priors and models. 
The choice of the prior is subjective and equivalent to choosing a normalised luckiness function in refined MDL \cite{grunwald2007minimum}.

\section{Proofs}
\label{app:sec:proofs}

We re-state the propositions and provide their proofs in this section.
We being by proving a lemma that will help us prove the propositions.
\begin{lemma}
    Given two Bayesian causal models $(\xcausey, \pi_{\nomodelxcausey})$ and $(\ycausex, \pi_{\nomodelycausex})$, they are Bayesian distribution equivalent precisely when
    \begin{align}
        \delta_{\mathcal{C}_{\nomodelxcausey}}\# \pi_{\nomodelxcausey} = \delta_{\mathcal{C}_{\nomodelycausex}}\# \pi_{\nomodelycausex}.
    \end{align}
\label{lemma:app:savage}
\end{lemma}
\begin{proof}
To prove this, we make use of the Hewitt-Savage representation theorem in \citet[thm. 9.4]{hewitt1955symmetric}.
From this, it holds that for each model, the data distribution is uniquely expressible in the form
\begin{align}
    P(\cdot \mid \xcausey) &= \int_{\mathcal{P}(\mathcal{X} \times \mathcal{Y})} \mu_{\nomodelxcausey}(\calcd Q) \cdot Q^{\otimes \infty}\,, \\
    P(\cdot \mid \ycausex) &= \int_{\mathcal{P}(\mathcal{X} \times \mathcal{Y})} \mu_{\nomodelycausex}(\calcd Q) \cdot Q^{\otimes \infty},
\end{align}
where $\mu_{\nomodelxcausey}$ and $\mu_{\nomodelycausex}$ are probability measures over the space $\mathcal{P}(\mathcal{X} \times \mathcal{Y})$, that is probability measures over the space of probability measures over $\mathcal{X} \times \mathcal{Y}$.
In particular $P(\cdot \mid \xcausey) =  P(\cdot \mid \ycausex)$ precisely when $\mu_{\nomodelxcausey} = \mu_{\nomodelycausex}$ \cite{hewitt1955symmetric}.
By construction we can take $\mu_{\nomodelxcausey} = \delta_{\mathcal{C}_{\nomodelxcausey}}\# \pi_{\nomodelxcausey}$ and $\mu_{\nomodelycausex} = \delta_{\mathcal{C}_{\nomodelycausex}}\# \pi_{\nomodelycausex}$.
The result follows.
\end{proof}

In words, the substance of this result is that in order to check the property of Bayesian distribution-equivalence, it suffices to compare the marginal distribution of a single bivariate observation under each model. 
Note that this result is agnostic to the model parametrisation and its identifiability, as it concerns only marginal data distributions.

\begin{proposition}
Given two BCMs $(\xcausey, \pi_{\nomodelxcausey})$, $(\ycausex, \pi_{\nomodelycausex})$, suppose that there exists a subset $\mathcal{C}_{\Delta} \subset \mathcal{C}_{\nomodelxcausey}$ such that $\pi_{\nomodelxcausey}(\mathcal{C}_{\Delta}) > 0$, and $\delta_{\mathcal{C}_{\nomodelxcausey}}(\mathcal{C}_{\Delta}) \cap \mathcal{F}_{\nomodelycausex}$ is empty.
Then the two Bayesian causal models are not Bayesian distribution-equivalent.
\end{proposition}
\begin{proof}
Let $\mathcal{F}_{\Delta} := \delta_{\mathcal{C}_{\nomodelxcausey}}(\mathcal{C}_{\Delta})$ noting that $\mathcal{F}_{\Delta} \in \mathcal{F}_{\nomodelxcausey}$ but $\mathcal{F}_{\Delta} \notin \mathcal{F}_{\nomodelycausex}$.
There is no choice of prior $\pi_{\nomodelycausex}$ such that $
(\delta_{\mathcal{C}_{\nomodelycausex}}\# \pi_{\nomodelycausex})(\mathcal{F}_{\Delta}) = 
(\delta_{\mathcal{C}_{\nomodelxcausey}}\# \pi_{\nomodelxcausey})(\mathcal{F}_{\Delta})$.
The result immediately follows from \cref{lemma:app:savage}.
\end{proof}

Note that the same result as above holds if we assume that $\delta_{\mathcal{C}_{\nomodelxcausey}}(\mathcal{C}_{\Delta}) \cap \mathcal{F}_{\nomodelycausex}$ is non-empty, but $\mathcal{C}_{\Delta}$ no mass under $\pi_{\nomodelycausex}$. In this respect, the effect of placing hard restrictions on the set of conditionals can be mimicked by suitable prior design.

For distribution-equivalent models, we show that the ICM assumption must also hold in the anti-causal factorisation for the models to be Bayesian distribution-equivalent.
We formalise the reversibility of the ICM assumption by introducing separable-compatibility. 
\begin{definition}
Let $(\xcausey, \pi_{\nomodelxcausey})$, $(\ycausex, \pi_{\nomodelycausex})$ be two Bayesian causal models where the underlying causal models are distribution-equivalent, denoting $\gamma$ as the corresponding translation mapping $\gamma: \mathcal{C}_{\nomodelxcausey} \to \mathcal{C}_{\nomodelycausex}$ (in \cref{def:distr_equiv}). 
Say the two are \textbf{separable-compatible} if: 1) the pushforward $\gamma\#\pi_{\nomodelxcausey}$ is separable with respect to $\mathcal{C}_{\nomodelycausex}$, 2) $\gamma^{-1}\#\pi_{\nomodelycausex}$ is separable with respect to $\mathcal{C}_{\nomodelxcausey}$.
\end{definition}

\begin{proposition}
Fix two Bayesian causal models $(\xcausey, \pi_{\nomodelxcausey})$, $(\ycausex, \pi_{\nomodelycausex})$, where the underlying causal models are distribution-equivalent. The two Bayesian causal models are Bayesian distribution-equivalent only if they are separable-compatible.
\end{proposition}
\begin{proof}
Assume that $(\xcausey, \pi_{\nomodelxcausey})$ and $(\ycausex, \pi_{\nomodelycausex})$ are Bayesian distribution-equivalent.
Take $\mu_{\nomodelxcausey} = \delta_{\mathcal{C}_{\nomodelxcausey}}\# \pi_{\nomodelxcausey}$ and $\mu_{\nomodelycausex} = \delta_{\mathcal{C}_{\nomodelycausex}}\# \pi_{\nomodelycausex}$.
By \cref{lemma:app:savage} $\mu_{\nomodelxcausey} = \mu_{\nomodelycausex}$.
Recalling that $\xcausey$ and $\ycausex$ are distribution-equivalent, the translation mapping $\gamma$ satisfies $\delta_{\mathcal{C}_{\nomodelxcausey}} = \delta_{\mathcal{C}_{\nomodelycausex}} \circ \gamma$, also that the mapping $\delta_{\mathcal{C}_{\nomodelycausex}}$ is injective.
We can argue 
\begin{align}
    \delta_{\mathcal{C}_{\nomodelxcausey}}\# \pi_{\nomodelxcausey} &= \delta_{\mathcal{C}_{\nomodelycausex}}\# \pi_{\nomodelycausex}, \nonumber \\
    (\delta_{\mathcal{C}_{\nomodelycausex}} \circ \gamma)\# \pi_{\nomodelxcausey} &= \delta_{\mathcal{C}_{\nomodelycausex}}\# \pi_{\nomodelycausex}, \nonumber \\
    \gamma\# \pi_{\nomodelxcausey} &=  \pi_{\nomodelycausex} \nonumber,
\end{align}
that is, the pushforward $\gamma\# \pi_{\nomodelxcausey} $ is equal to $\pi_{\nomodelycausex}$, and hence is separable with respect to $\mathcal{C}_{\nomodelycausex}$.
To conclude, $(\xcausey, \pi_{\nomodelxcausey})$ and $(\ycausex, \pi_{\nomodelycausex})$ are separable-compatible.
\end{proof}

The proof of \cref{cor:parametrise_corrollary} directly follows from the above.
For parametrised models, we can define $\delta$ as the map from the parameter space $\Theta \times \Phi$ to the space of probability measures $\mathcal{P}(\mathcal{X} \times \mathcal{Y})$.
From the above proof, we can see that this will require the parametrisation to be injective, that is, the map $\delta : \Theta \times \Phi \to \mathcal{P}(\mathcal{X} \times \mathcal{Y})$ to be injective.

\section{Analysis of Models}
\label{sec:analysis_of_models}

In \cref{sec:asymmetry_theorem}, we showed that two models being separable-compatible is a necessary condition for the models to be Bayesian distribution-equivalent.
Thus, if the models are not separable-compatible, they will not be Bayesian distribution-equivalent.
In this section, we analyse specific models and argue that Bayesian causal models being separable-compatible is a strong condition that does not hold very often.

\subsection{Unnormalised linear Gaussian model}
\label{sec:non_identifiability_gaussian}

\begin{figure}[h!]
\begin{subfigure}{0.3\textwidth}
    \centering
  \tikz{
 \node[obs] (x) {$X_i$};%
 \node[obs,right=of x,xshift=0.25cm] (y) {$Y_i$};
 \node[latent,left=of x,xshift=-0.15cm,yshift=0.5cm] (a_0) {$a_0$};
  \node[latent,left=of x,xshift=-0.15cm,yshift=-0.5cm] (sigma_0) {$\sigma_0$};
 \node[latent,right=of y,xshift=0.15cm,yshift=0.5cm] (a_1) {$a_1$};
  \node[latent,right=of y,xshift=0.15cm,yshift=-0.5cm] (sigma_1) {$\sigma_1$}; \plate[inner sep=0.3cm, xshift=0cm, yshift=0.12cm] {plate1} {(x) (y)} {$i = 1,\ldots, N$}; 
 \edge {x} {y};
 \edge {a_0} {x};
 \edge {sigma_0} {x};
 \edge {a_1} {y} 
\edge {sigma_1} {y} 
 }
 \caption{}
\end{subfigure}
\hspace{0.5cm}
\begin{subfigure}{0.3\textwidth}
    \centering
  \tikz{
 \node[obs] (x) {$X_i$};%
 \node[obs,right=of x,xshift=0.25cm] (y) {$Y_i$};
 \node[latent,above=of x,xshift=-0.5cm,yshift=0.15cm] (a_0) {$a_0$};
  \node[latent,above=of x,xshift=0.5cm,yshift=0.15cm] (sigma_0) {$\sigma_0$};
 \node[latent,above=of y,xshift=-0.5cm,yshift=0.15cm] (a_1) {$a_1$};
  \node[latent,above=of y,xshift=0.5cm,yshift=0.15cm] (sigma_1) {$\sigma_1$}; \plate[inner sep=0.3cm, xshift=0cm, yshift=0.12cm] {plate1} {(x) (y)} {$i = 1,\ldots, N$}; 
 \edge {y} {x};
 \edge {a_0} {x};
 \edge {sigma_0} {x};
 \edge {a_1} {y} 
\edge {sigma_1} {y} 
 \edge {a_0} {y};
 \edge {sigma_0} {y};
 \edge {a_1} {x} 
\edge {sigma_1} {x} 
 }
 \caption{}
\end{subfigure}
\hspace{-.5cm}
\begin{subfigure}{0.3\textwidth}
    \centering
  \tikz{
 \node[obs] (x) {$X_i$};%
 \node[obs,right=of x,xshift=0.25cm] (y) {$Y_i$};
 \node[latent,left=of x,xshift=-0.15cm,yshift=0.5cm] (a_0) {$a_0^{\prime}$};
  \node[latent,left=of x,xshift=-0.15cm,yshift=-0.5cm] (sigma_0) {$\sigma_0^{\prime}$};
 \node[latent,right=of y,xshift=0.15cm,yshift=0.5cm] (a_1) {$a_1^{\prime}$};
  \node[latent,right=of y,xshift=0.15cm,yshift=-0.5cm] (sigma_1) {$\sigma_1^{\prime}$}; \plate[inner sep=0.3cm, xshift=0cm, yshift=0.12cm] {plate1} {(x) (y)} {$i = 1,\ldots, N$}; 
 \edge {y} {x};
 \edge {a_0} {x};
 \edge {sigma_0} {x};
 \edge {a_1} {y} 
\edge {sigma_1} {y} 
 }
 \caption{}
\end{subfigure}
\caption{
Graphical models for: (a) The linear Gaussian causal model $\xcausey$ in \cref{eq:app:linear_gauss_causalfact}.
(b) The anti-causal factorisation of $\xcausey$ in \cref{eq:app:linear_gauss_anticausalfact}.
(c) The causal model for $\ycausex$, where ICM holds in the factorisation $P(Y)P(X|Y)$.
}
\label{fig:app:lineargaussian}
\end{figure}

Here, we consider the Linear Gaussian model of the form 
\begin{align}
    P(X| a_0, \sigma_0, \xcausey) &= \mathcal{N}(a_0, \sigma_0^2), \nonumber \\
    P(Y|X, a_1, \sigma_1, \xcausey) &= \mathcal{N}(a_1 X, \sigma_1^2).
    \label{eq:app:linear_gauss_causalfact}
\end{align}
We can compute the anti-causal factorisation of this model as
\begin{align}
    P(Y | a_0, \sigma_0, a_1, \sigma_1,\xcausey) &= \mathcal{N}(a_1a_0, \sigma_1^2 + a_1^2 \sigma_0^2), \label{eq:app:indep_anticausal} \\
    P(X |Y, a_0, \sigma_0, a_1, \sigma_1,\xcausey) &= \mathcal{N}\left(\Sigma \left( \frac{a_1}{\sigma_1^2} Y + \frac{a_0}{\sigma^2_0} \right) , \Sigma \right)
    \label{eq:app:linear_gauss_anticausalfact}
\end{align}
with $\Sigma = \frac{\sigma_0^2 \sigma_1^2}{ \sigma_1^2 + \sigma_0^2 a_1^2}$.
We can see that while the causal factorisation factorises as \cref{fig:app:lineargaussian} (a), the anti-causal factorisation, in general, factorises as \cref{fig:app:lineargaussian} (b).
The only way for the linear Gaussian Bayesian causal models to be separable-compatible, is if the priors on the parameters are chosen such that 
\begin{align}
    a_1a_0 &\indep \Sigma \left( \frac{a_1}{\sigma_1^2} Y + \frac{a_0}{\sigma^2_0} \right), \\
    \sigma_1^2 + a_1^2 \sigma_0^2 &\indep \Sigma.
\end{align}
If the above holds, we can reparametrise the anti-causal factorisation to factorise as in \cref{fig:app:lineargaussian} (c).
\citet{geiger2002parameter} show that the \textbf{only} prior that satisfies this property is the normal-Wishart prior (the Wishart prior can be scaled by a real function for the bivariate case, see Appendix of \citet{geiger2002parameter}).
This is known as the BGe model.

\paragraph{BGe Model:} We illustrate how the prior in the BGe model gives the independence required in \cref{eq:app:indep_anticausal}.
We only show this for the mean and defer to the results in \citet{geiger2002parameter} for the complete argument.
In this model, a Normal-Wishart prior is placed directly on the joint Gaussian
\begin{align}
    P\left(\begin{bmatrix}X \\ Y\end{bmatrix} \right) = \mathcal{N}\left( \begin{bmatrix} \mu_{1}\\
\mu_{2} \end{bmatrix}, \begin{bmatrix}W_{11} & W_{12}\\ W_{21} & W_{22}
  \end{bmatrix}^{-1} \right).
\end{align}
A normal-Wishart prior on $\bm{\mu}$ is of the form $P(\bm{\mu})= \mathcal{N}(\bm{\eta}, \gamma \mathbf{W}^{-1})$.
We can factorise the distribution for $X$ and $Y$, for example
\begin{align}
P(X) &= \mathcal{N} \left( \mu_{1}, (W_{11} - W_{12}W_{22}^{-1}W_{21})^{-1}\right), \\
P(Y| X) &= \mathcal{N}(\mu_{2} - W_{21}^{-1}W_{22}X +  W_{21}W_{22}^{-1}\mu_{1}, W_{22}^{-1}).
\end{align}
We show that the random variable $\mu_{1}$ is independent of $\mu_{2} + W_{21}W_{22}^{-1}\mu_{1}$. 
Hence, separability of priors holds. 
To see this, we can factorise the prior 
\begin{align*}
p(\mu_{1}) &=  \mathcal{N}(\eta_{1}, (W_{11} - W_{12}W_{22}^{-1}W_{21})^{-1}),\\
p(\mu_{2}|\mu_{1}) &= \mathcal{N}(\eta_{2} - W_{21}W_{22}^{-1}(\mu_{1} - \eta_{1}), W_{22}^{-1})
\end{align*}
The implied distribution of $\mu'_{2} = \mu_{2} + W_{21}W_{22}^{-1}\mu_{1}$ is thus
\begin{align*}
p_{\mu_{2}| \mu_{1}}(\mu_{2} + W_{21}W_{22}^{-1}\mu_{1}) &= \mathcal{N}(\eta_{2} + W_{21}W_{22}^{-1}\eta_{1}, W_{22}^{-1}),
\end{align*}
which is readily seen to be independent of $\mu_1$.
The same can be shown for the Wishart distribution on $\mathbf{W}$, i.e. that $W_{11} - W_{12}W_{22}^{-1}W_{21}$ is independent of $W_{22}$ \cite{geiger2002parameter}.

Thus for any other choice of prior, if the prior matches the data generating distribution, we expect Bayesian model selection to find the correct causal direction in this case.
However, this requires knowledge of the true variance of the cause and effect in this case.
This matches known results in \citet{loh2014high}.
The above may not be desirable as simple scaling of the data, changing the variance, may alter inference of the causal direction, see discussion in  \citet{reisach2021beware}.
We thus suggest normalising all datasets which will render the method mean and scale invariant.

We can also visually verify that two Bayesian Linear Gaussian models imply different data distributions.
We draw parameters from priors ---  an inverse gamma for the scales.
For ease of exposition, we consider the case where the means are 0.
Of course, the datasets will differ more if we do not include this constraint.
For the same set of drawn parameters, we also find the parameters of the joint for the causal model $\ycausex$.
We then plot the contours of the resulting joint distributions as generated by the two causal models.
This shows the likely joint distributions that a model generates.
 \Cref{fig:nonnormalised-gaussian-samples} shows contours of such Gaussians.
This shows distributions with the same mean and variances for the cause and effect but with  different ground truth causal directions $\nomodelxcausey$(red) and $\nomodelycausex$ (blue). 
 \Cref{fig:nonnormalised-gaussian-samples} (a) shows 1 such joint, and (b) shows 10 such joints.
From these figures, it is clear to see that the causal directions alone imply different joint distributions.

\begin{figure*}[h!]
    \centering
    \begin{subfigure}{0.48\textwidth}
    \resizebox{\linewidth}{!}{
        \includegraphics[width=\textwidth]{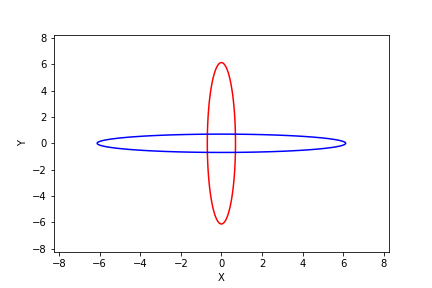}
    }
    \caption{}
    \end{subfigure}
    \begin{subfigure}{0.48\textwidth}
    \resizebox{\linewidth}{!}{
        \includegraphics[width=\textwidth]{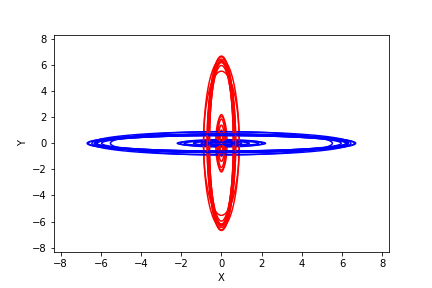}
    }
    \caption{}
    \end{subfigure}
    \caption{Shows samples of joint distributions from the same priors on parameters of the cause and effect, but with different causal models. Red shows the joint of the causal model $\nomodelxcausey$ and blue shows the joint of the causal model $\nomodelycausex$. The contours are plotted for the same draw of parameters, showing that the different causal models will explain different joint distributions well. }
    \label{fig:nonnormalised-gaussian-samples}
\end{figure*}

\subsection{Normalised linear Gaussian model}

As we suggest normalising the dataset in line with \citet{reisach2021beware}, we show here that for this model, if the same prior is used for both causal models (as we recommend for our model choice), then the causal models will be Bayesian distribution-equivalent.

We conduct the same procedure as above, but ensure that the marginal distributions of $X$ and $Y$ are normalised to $\mathcal{N}(0,1)$.
\Cref{fig:normalised-gaussian-samples} (a) shows the contours of one such joint distribution for the two causal models $\nomodelxcausey$ and $\nomodelycausex$, while (b) shows 10 such joints.
We can see that the joints completely overlap for the two causal models, hence the Bayesian model selection will have no opinion on the true causal model, and will convey this uncertainty.

We can show this mathematically by assuming the data samples has been generated as follows
\begin{align}
    \Pi(X| a_0, \sigma_0) &= \mathcal{N}(a_0, \sigma_0^2),\\
    \Pi(Y|X, a_1, \sigma_1) &= \mathcal{N}(a_1 X, \sigma_1^2).
\end{align}
On normalisation, we create two new variables $X'$ and $Y'$ that have a standard normal distribution. Accordingly, we set $X' = \frac{X - a_0}{\sigma_0}$ and $Y' = \frac{Y - a_0a_1}{\sqrt{\sigma_1^2 + \sigma_0^2 a_1^2}}$, with distributions
\begin{align}
    \Pi(X') &= \mathcal{N}(0, 1),\\
    \Pi(Y'|X', a_1, \sigma_0, \sigma_1) &= \mathcal{N} \left( \frac{\sigma_0a_1}{\sqrt{\sigma_1^2 + \sigma_0^2 a_1^2}} X', \frac{\sigma_1^2}{\sigma_1^2 + \sigma_0^2 a_1^2} \right).
\end{align}
We can see that both $\Pi(X')$ and $\Pi(Y')$ are standard normal.
The anti-causal factorisation in this case results in the exact same distribution. Using simple algebra we get
\begin{align}
    \Pi(Y') &= \mathcal{N}(0, 1),\\
    \Pi(X'|Y',a_1, \sigma_0, \sigma_1) &= \mathcal{N} \left( \frac{\sigma_0a_1}{\sqrt{\sigma_1^2 + \sigma_0^2 a_1^2}} Y', \frac{\sigma_1^2}{\sigma_1^2 + \sigma_0^2 a_1^2} \right).
\end{align}
For any prior on the parameters of $\Pi(X') \Pi(Y'\mid X' a_1, \sigma_0, \sigma_1)$, the same prior on $\Pi(Y') \Pi(X'\mid Y' a_1, \sigma_0, \sigma_1)$ will result in the two causal models having the same data distribution.
Hence, there exist priors such that the Bayesian causal models constructed out of these causal models cannot be discriminated by the marginal likelihood.
Note, we can clearly see here that for any choice of priors, the two models are separable-compatible.

\begin{figure*}[h!]
    \centering
    \begin{subfigure}{0.48\textwidth}
    \resizebox{\linewidth}{!}{
        \includegraphics[width=\textwidth]{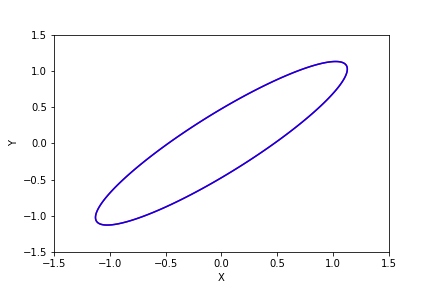}
    }
    \caption{}
    \end{subfigure}
    \begin{subfigure}{0.48\textwidth}
    \resizebox{\linewidth}{!}{
        \includegraphics[width=\textwidth]{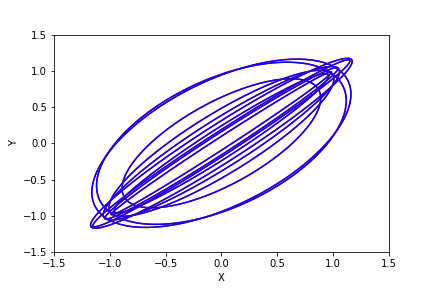}
    }
    \caption{}
    \end{subfigure}
    \caption{Best viewed in colour. Joint contours for normalised cause and effect. The joint contours completely overlap for the two causal models. Hence, both causal models will explain both datasets equally well.}
    \label{fig:normalised-gaussian-samples}
\end{figure*}

\subsection{Gaussian Process Latent variable model}

We refer the reader to \citet{rasmussen2003gaussian} for an introduction to Gaussian processes.
Here, we argue that the GPLVM prior does not satisfy separability with respect to the anti-causal factorisation.
We only condition on the model ($\xcausey$ or $\ycausex$) on the left hand side of equations to avoid clutter of notation.
Note that the final data density (density implied by the data distribution) is found by integrating over the hyperparameters as laid out in \cref{sec:evidence_approx}. 
For clarity, we omit this additional step from this section.

\paragraph{Model $\xcausey$: } 
In $\xcausey$, $\bfx$ is modelled directly, and $\bfy$ is modelled conditional on $\bfx$.
For the estimation of $\bfx$ we have 
\begin{align}
    p(\bfx, \mathbf{f}_X, \mathbf{w}_X| \lambda_X, \xcausey) = p(\bfx|\mathbf{f}_X)p(\mathbf{f}_X| \mathbf{w}_X, \lambda_X) p(\mathbf{w}_X),
\end{align}
where $\lambda_X$ collects all the hyperparameters for modelling $\bfx$.
The data density can be found by integrating over the priors 
\begin{align}
    p(\bfx| \lambda_X, \xcausey) &=\iint p(\bfx|\mathbf{f}_X)p(\mathbf{f}_X| \mathbf{w}_X, \lambda_X) p(\mathbf{w}_X) d\mathbf{f}_X d\mathbf{w}_X \\
    &= \int   p(\bfx|\mathbf{w}_X, \lambda_X) p(\mathbf{w}_X) d\mathbf{w}_X,
\end{align}
where $ p(\bfx| \mathbf{w}_X,\lambda_X) = \mathcal{N}(0, K_{\lambda_X}(\mathbf{w}_X, \mathbf{w}_{X}') + \sigma^2_X)$, $\sigma^2_X$ is the likelihood noise hyperparameter and $K_{\lambda_X}$ is the chosen kernel with hyperparameters $\lambda_X$.
All modelling choices are as laid out in section 5.
$p(\mathbf{w}_X)$ is usually chosen to be standard Gaussian distribution.
From this, we can see that the dataset density is a mixture of Gaussians, mixed by Gaussian-distributed weights.
Similarly, for the conditional model $\bfy|\bfx$ we have
\begin{align}
    p(\bfy| \bfx, \lambda_Y, \xcausey) &=\iint  p(\bfy|\bfx, \mathbf{f}_Y)p(\mathbf{f}_Y| \mathbf{w}_Y, \lambda_Y) p(\mathbf{w}_Y) d\mathbf{f}_Y d\mathbf{w}_Y \\
    &= \int   p(\bfy|\bfx, \mathbf{w}_Y, \lambda_Y) p(\mathbf{w}_Y) d\mathbf{w}_Y,
\end{align}
where $ p(\bfy|\bfx, \mathbf{w}_Y, \lambda_Y) = \mathcal{N}(0, K_{\lambda_Y}((\bfx,\mathbf{w}_Y), (\bfx, \mathbf{w}_{Y})') + \sigma^2_Y)$, and $p(\mathbf{w}_Y)$ is a standard Gaussian.
Hence 
\begin{align}
    p(\bfx| \lambda_X, \xcausey) p(\bfy| \bfx, \lambda_Y, \xcausey) &= \int   p(\bfx|\mathbf{w}_X, \lambda_X) p(\mathbf{w}_X) d\mathbf{w}_X  \int   p(\bfy|\bfx, \mathbf{w}_Y, \lambda_Y) p(\mathbf{w}_Y) d\mathbf{w}_Y.
\end{align}

\paragraph{Model $\ycausex$: } Similarly, for the causal model $\ycausex$, we have 
\begin{align}
    p(\bfy| \lambda_X, \ycausex) p(\bfx| \bfy, \lambda_Y, \ycausex) &= \int   p(\bfy|\mathbf{w}_Y, \lambda_Y) p(\mathbf{w}_Y) d\mathbf{w}_Y  \int   p(\bfx|\bfy, \mathbf{w}_X, \lambda_X) p(\mathbf{w}_X) d\mathbf{w}_X,
    \label{app:eq:gplvm_ycausex_causal}
\end{align}
where 
\begin{align}
    p(\bfy| \mathbf{w}_Y, \lambda_Y, \ycausex) &= \mathcal{N}(0, K_{\lambda_Y}(\mathbf{w}_Y, \mathbf{w}_{Y}') + \sigma^2_Y), \nonumber \\
    p(\bfx|\bfy, \mathbf{w}_X, \lambda_X, \ycausex) &= \mathcal{N}(0, K_{\lambda_X}((\bfy,\mathbf{w}_X), (\bfy, \mathbf{w}_{X})') + \sigma^2_X). \label{app:eq:gplvm_backward_densities}
\end{align}

To compare against the model $\xcausey$, we must find the anti-causal factorisation of $\ycausex$.
First, we can see that if we factorise \cref{app:eq:gplvm_ycausex_causal}, then the priors on the distributions are no longer separable and hence the integrals cannot be estimated separately
\begin{align}
     &\int   p(\bfy|\mathbf{w}_Y, \lambda_Y) p(\mathbf{w}_Y) d\mathbf{w}_Y  \int   p(\bfx|\bfy, \mathbf{w}_X, \lambda_X) p(\mathbf{w}_X) d\mathbf{w}_X \\
     &= \iint   p(\bfx|\mathbf{w}_Y,\mathbf{w}_X, \lambda_X, \lambda_Y)   p(\bfy|\bfx, \mathbf{w}_Y,\mathbf{w}_X, \lambda_X, \lambda_Y)p(\mathbf{w}_Y) p(\mathbf{w}_X) d\mathbf{w}_Y  d\mathbf{w}_X.
\end{align}
Here,
\begin{align}
    p(\bfx| \mathbf{w}_X, \mathbf{w}_Y, \lambda_X, \lambda_Y, \ycausex) &= \int p(\bfy| \mathbf{w}_Y, \lambda_Y)  p(\bfx|\bfy, \mathbf{w}_X, \lambda_X)    d\bfy , \\
    p(\bfy| \bfx,  \mathbf{w}_X, \mathbf{w}_Y, \lambda_X, \lambda_Y, \ycausex) &= \frac{p(\bfy| \mathbf{w}_Y, \lambda_Y)  p(\bfx|\bfy, \mathbf{w}_X, \lambda_X)}{p(\bfx| \mathbf{w}_X, \mathbf{w}_Y, \lambda_X, \lambda_Y)}.
\label{eq:app:anti-causal_factorisation}
\end{align}
As the kernel $K$ is generally complicated non-linear function, these terms are not Gaussian distributed, in contrast with the terms of the causal factorisation of $\xcausey$.

Directly comparing the marginals over $\bfx$ in the two causal models,
\begin{align}
    p(\bfx| \lambda_X, \xcausey) &= \int   p(\bfx|\mathbf{w}_X, \lambda_X) p(\mathbf{w}_X) d\mathbf{w}_X, \\
    p(\bfx| \lambda_X, \lambda_Y, \ycausex) &= \iint \left( \int p(\bfy| \mathbf{w}_Y, \lambda_Y)  p(\bfx|\bfy, \mathbf{w}_X, \lambda_X)    d\bfy \right) p(\mathbf{w}_Y) p(\mathbf{w}_X) d\mathbf{w}_Y  d\mathbf{w}_X  \nonumber,
\end{align}
while the term for $\xcausey$ is a mixture of Gaussians, the term for $\ycausex$ is clearly not a mixture of Gaussians.
The same holds for the conditionals
\begin{align}
    p(\bfy| \bfx, \lambda_Y, \xcausey) &= \int   p(\bfy|\bfx, \mathbf{w}_Y, \lambda_Y) p(\mathbf{w}_Y) d\mathbf{w}_Y, \\
     p(\bfy| \bfx, \lambda_X, \lambda_Y, \ycausex) &= \iint \frac{p(\bfy| \mathbf{w}_Y, \lambda_Y)  p(\bfx|\bfy, \mathbf{w}_X, \lambda_X)}{p(\bfx| \mathbf{w}_X, \mathbf{w}_Y, \lambda_X, \lambda_Y)} p(\mathbf{w}_Y) p(\mathbf{w}_X) d\mathbf{w}_Y  d\mathbf{w}_X  .\nonumber
\end{align}
As there is no analytical form for the terms of the anti-causal factorisation of $\ycausex$, we proceed by explicitly showing that the variance of $p(\bfx| \mathbf{w}_X, \mathbf{w}_Y, \lambda_X, \lambda_Y, \ycausex)$ depends on $\mathbf{w}_X, \mathbf{w}_Y$ for a certain choice of kernel (thus that the term actually depends on $\mathbf{w}_X, \mathbf{w}_Y$). 
We then contend that it is unlikely that the variance of $p(\bfy| \bfx,  \mathbf{w}_X, \mathbf{w}_Y, \lambda_X, \lambda_Y, \ycausex)$ also does not depend on $\mathbf{w}_X, \mathbf{w}_Y$.
If this holds, then the induced prior over $\mathbf{w}_X, \mathbf{w}_Y$  is not separable with respect to the anti-causal factorisation.

\paragraph{Explicit Derivation for RBF kernels.}
From \cref{eq:app:anti-causal_factorisation}, we can see that both terms of the anti-causal factorisation of $\ycausex$ depend on $\mathbf{w}_X$, $\mathbf{w}_Y$ and hence the prior over these terms is not separable.
That is, while the causal factorisation factorises as \cref{fig:ICM_graphical_model}(b), the anti-causal factorisation does not factorise as \cref{fig:ICM_graphical_model}(a).
We can show the explicit dependence on $\mathbf{w}_X$, $\mathbf{w}_Y$ to one of the terms in \cref{eq:app:anti-causal_factorisation} for certain kernel choices --- specifically the ARD RBF kernel.
We can find the variance of 
\begin{align}
    p(\bfx| \mathbf{w}_X, \mathbf{w}_Y, \lambda_X, \lambda_Y, \ycausex) &= \int p(\bfy| \mathbf{w}_Y, \lambda_Y)  p(\bfx|\bfy, \mathbf{w}_X, \lambda_X) d\bfy. 
\end{align}
Here (suppressing the $\lambda$ and model notation for clarity), 
\begin{align}
    \mathbb{E}[X^2 \mid W_X, W_Y] &= \int p(\bfy | \mathbf{w}_Y) \int x^2 p(\bfx | \bfy, \mathbf{w}_X) \calcd \bfx \calcd \bfy \\
    &= \int p(\bfy | \mathbf{w}_Y) \left[\mathbf{k}(\mathbf{y}, \mathbf{y}^{\prime})\mathbf{k}(\mathbf{w}_X, \mathbf{w}_{X}^{\prime}) + \sigma_X^2 \right] \calcd \bfy.
\end{align}
We can calculate the kernel expectation 
\begin{align}
\mathbf{\Psi} := \int p(\bfy | \mathbf{w}_Y) \mathbf{k}(\mathbf{y}, \mathbf{y}^{\prime}) \calcd \bfy,
\end{align}
where $\mathbf{\Psi} \in \mathbb{R}^{N \times N}$. 
Let $\mathbf{k}_{nm}$ be the $2 \times 2$ matrix of the form (covariance of $p(y_n, y_m |w_{Y_n}, w_{Y_m} )$)
\begin{align}
    \mathbf{k}_{nm} = \begin{pmatrix}
        k(w_{Y_n}, w_{Y_n}) & k(w_{Y_n}, w_{Y_m}) \\ k(w_{Y_m}, w_{Y_n}) & k(w_{Y_m}, w_{Y_m})
    \end{pmatrix} = \begin{pmatrix}
        k_{11} & k_{12} \\ k_{21} & k_{22}
    \end{pmatrix}. 
\end{align}
We can write the $nm^{th}$ element of $\mathbf{\Psi}$ as 
\begin{align}
    [\mathbf{\Psi}]_{nm} &= \iint p(y_n, y_m |w_{Y_n}, w_{Y_m} ) k(y_n, y_m) \calcd y_n \calcd y_m \\
    &= \int p(y_m | w_{Y_m}) \int p(y_n | y_m, w_{Y_n}, w_{Y_m} ) k(y_n, y_m) \calcd y_m \calcd y_n,
\end{align}
where
\begin{align}
    p(y_m | w_{Y_m}) &= \mathcal{N}(0, k_{22}) \\
    p(y_n | y_m, w_{Y_n}, w_{Y_m} ), &= \mathcal{N}(k_{12}k_{22}^{-1}y_m, k_{11} - k_{12}k_{22}^{-1}k_{21}).
\end{align}
We let $\mu = k_{12}k_{22}^{-1}$ and $\Sigma = k_{11} - k_{12}k_{22}^{-1}k_{21}$ in the following.
We can write the terms inside the exponential of $p(y_n | y_m, w_{Y_n}, w_{Y_m} ) k(y_n, y_m)$ as (ignoring the lengthscale of the RBF kernel)
\begin{align}
    &\propto \exp \left(  -\frac{1}{2} \left( (y_n - \mu y_m)^2 \Sigma^{-1} + (y_n - y_m)^2  \right) \right) \\
    &=\exp\left(- \frac{1}{2} c\right)\exp \left(  -\frac{1}{2} \left( (y_n - \nu)^2 \eta^{-1} \right) \right),
\end{align}
where $\eta^{-1} = (\Sigma^{-1} + 1)$, $\nu = \eta (\mu y_m \Sigma^{-1} + y_m)$, and $c = \eta (y_m^2 \mu^2 \Sigma^{-1} + y_m^2 \Sigma^{-1} - 2\mu y_m^2 \Sigma^{-1})$.
This gives that
\begin{align}
    [\mathbf{\Psi}]_{nm} = \frac{\sqrt{\eta}}{| \mathbf{k}_{nm} |^{0.5}} \int \exp\left(- \frac{1}{2} c\right) p(y_m | w_{Y_m}) \calcd y_m. 
\end{align}
We again expand the terms inside the exponential of the integrand,
\begin{align}
    \propto \exp \left( - \frac{1}{2} y^2_m ( k_{22}^{-1 } + \eta \mu^2 \Sigma^{-1} + \eta \Sigma^{-1} - 2 \mu \eta \Sigma^{-1}) \right).
\end{align}
Integrating this gives
\begin{align}
    [\mathbf{\Psi}]_{nm} = \frac{\sqrt{ 2 \eta}}{| \mathbf{k}_{nm} |^{0.5} \sqrt{k_{22}^{-1 } + \eta \mu^2 \Sigma^{-1} + \eta \Sigma^{-1} - 2 \mu \eta \Sigma^{-1})}}.
\end{align}
Hence, we can explicitly see that the variance of $p(\bfx| \mathbf{w}_X, \mathbf{w}_Y, \lambda_X, \lambda_Y, \ycausex)$ depends on both $\mathbf{w}_Y$ and $\mathbf{w}_X$.
The variance of $p(\bfy| \bfx,  \mathbf{w}_X, \mathbf{w}_Y, \lambda_X, \lambda_Y, \ycausex)$ is much harder to calculate.
We view it as highly unlikely that it would marginally depend on neither of $\mathbf{w}_Y$ and $\mathbf{w}_X$.
If it depends on either $\mathbf{w}_Y$ or $\mathbf{w}_X$, then the anti-causal factorisation shares parameters, and the prior is thus unlikely to be separable in the anti-causal direction.

\section{Derivations for section 4.3} 
\label{sec:derivation_poe}

\subsection{Derivation of the probability of error of both models being the same under similar priors}
Here we show that the probability of error of both the models is exactly the same.
Assuming a parametrisation and denoting parameters as $\phi$ for $X$ and $\theta$ for $Y$, we get
\begin{align}
    p(\bfx, \bfy|\xcausey) &= \int p(\bfx | \phi, \xcausey) \pi(\phi| \xcausey) \calcd\phi  \int p(\bfy|\bfx, \theta, \xcausey) \pi(\theta| \xcausey) \calcd\theta \,, 
    \label{eq:bayes_causal_model_causal_app}\\
    p(\bfx, \bfy|\ycausex) &= \int p(\bfy | \theta, \ycausex) \pi(\theta| \ycausex)\calcd\phi \int p(\bfx|\bfy, \phi, \ycausex) \pi(\phi| \ycausex) \calcd \theta  \,. \label{eq:bayes_causal_model_anti-causal_app}
\end{align}
The marginal and conditional densities in the above are chosen from the same families.
If we assume similar priors on the cause and effect in both models, then $\pi(\phi| \xcausey) = \pi(\theta| \ycausex)$ and $\pi(\theta| \xcausey)=\pi(\phi| \ycausex)$.
Thus, $p(\bfx, \bfy|\xcausey) = p(\bfy, \bfx|\ycausex)$.
That is, swapping $\bfx$ and $\bfy$ in one model will give the dataset density of the other model.
We can use this to show 
\begin{align}
    P(E| \xcausey) &= \int_{\mathcal{R}} p(\bfx, \bfy| \xcausey ) \calcd (\bfx, \bfy) \,, && \mathcal{R} = \{(\bfx, \bfy) \suchthat p(\bfx, \bfy|\ycausex) > p(\bfx, \bfy|\xcausey)\} \\
    &= \int_{\mathcal{R}} p(\bfy, \bfx| \ycausex ) \calcd (\bfx, \bfy) \,, && \mathcal{R} = \{(\bfx, \bfy) \suchthat p(\bfy, \bfx|\xcausey) > p(\bfy, \bfx|\ycausex)\} \\
    &=  P(E| \ycausex). 
\end{align}

\subsection{Upper bound on standard deviation of probability of error}
The probability of error is
\begin{align}
    P(\text{Error} )  &= P(\text{Error}| \xcausey) \\
    &= \int \mathbb{I}[\data \in \mathcal{R}]  \cdot p(\data| \xcausey )  \calcd \data \\
    \hat{I} &=  \frac{1}{T} \sum_{t=1}^T \mathbb{I}[\data_t \in \mathcal{R}], \ \ \data_t \sim p(\data| \xcausey ),
\end{align}
where $\mathbb{I}[\cdot]$ is the indicator function and $\hat{I}$ is the Monte Carlo estimator of $ P(\text{Error} ) $.
We can bound the variance by using the fact that $\text{Var}(\mathbb{I}[\data_t \in \mathcal{R}]) \leq 0.25$, upon viewing $\mathbb{I}[\cdot]$ as a Bernoulli random variable.
We hence obtain that
\begin{align}
    \text{Var}(\hat{I}) &\leq \frac{1}{4T}.
\end{align}
This can be calculated numerically.

\subsection{Derivation of total probability of error}
\label{sec:app:total_prob_error}

The total probability of error is (see \citet{nielsen2014generalized} for more details)
\begin{align}
    P(E) &= P(E | \xcausey)P(\xcausey) + P(E|\ycausex)P(\ycausex) \\
    &= \int_{\mathcal{R}_Y} p(\data| \xcausey)p(\xcausey) \calcd \data + \int_{\mathcal{R}_X} p(\data| \ycausex)p(\ycausex) \calcd \data \\
    &= \int \min(p(\data| \xcausey)p(\xcausey), p(\data| \ycausex)p(\ycausex)) \calcd \data,
\end{align}
where $\mathcal{R}_Y = \{\data \suchthat p(\data|\ycausex) > p(\data|\xcausey)\}$, and $\mathcal{R}_X = \{\data \suchthat p(\data|\ycausex) < p(\data|\xcausey)\}$.
We now use $\min(a,b) = \frac{a + b - | b-a |}{2}$ to obtain
\begin{align}
    P(E) &= \frac{1}{2} - \frac{1}{2} \int |  p(\data|\xcausey) - p(\data|\ycausex)  | \calcd \data \\
    &=  \frac{1}{2} (1-  \text{TV}[P(\cdot| \xcausey ),  P(\cdot| \ycausex )]).
\end{align}

\subsection{Misspecification}
\label{sec:app:prob_error_misspecify}
We can write
\begin{align}
    | Q(E) - P(E) | &= \frac{1}{2} \left| \text{TV}(P(\cdot| \xcausey ),  P(\cdot| \ycausex )) - \text{TV}(Q(\cdot| \xcausey ),  Q(\cdot| \ycausex )  \right| \\
    &= \frac{1}{2} \left| \int \left( \left| p(\data| \xcausey ) -   p(\data| \ycausex )  \right| - \left| q(\data| \xcausey ) -   q(\data| \ycausex )  \right| \right) \calcd \data  \right| \\
    &\leq \frac{1}{2} \left| \int \left( \left| p(\data| \xcausey ) -  q(\data| \xcausey ) + q(\data| \ycausex ) -  p(\data| \ycausex )   \right| \right) \calcd \data  \right| \\
    &\leq \frac{1}{2} \left| \int \left( \left| p(\data| \xcausey ) -  q(\data| \xcausey ) \right| + \left| q(\data| \ycausex ) -  p(\data| \ycausex )   \right| \right) \calcd \data  \right| \\
    &= \frac{1}{2} \left|  \text{TV}(P(\cdot| \xcausey ),  Q(\cdot| \xcausey )) + \text{TV}(P(\cdot| \ycausex ),  Q(\cdot| \ycausex ))  \right|.
\end{align}
Where we use the reverse triangle inequality and the triangle inequality.
We can get rid of the absolute value by noticing that TV is bounded by $0$ and $1$.

\section{Model Details}
\label{app:gplvm_model_details}

Here, we provide a more in depth introduction to our model and approximations.

\subsection{Latent variable Gaussian Processes with inducing points}
\label{sec:inducing_point_gplvm}

Gaussian processes (GPs) \cite{rasmussen2003gaussian} are non-parametric Bayesian models that define a prior over functions.
The form of the prior is controlled by  choice of a kernel function, $\mathbf{K}$.
Specifically, the kernel defines a covariance over outputs for the function.
The kernels are parametrised by continuous hyperparameters.
Adding kernels together and varying their hyperparameters allows for construction of flexible priors which support a wide range of functions.

Latent variable Gaussian Processes (GPLVM) consider a latent noise term $\mathbf{w}$ as an input with an associated prior.
Integrating over the noise term allows for modelling heteroscedastic noise as well as non-Gaussian likelihoods.
GPs have a well-known computational cost of $\mathcal{O}(N^3)$ where $N$ is the number of samples, which prohibits their direct application to large datasets.
To allow for scalability, we use an inducing point approximation \cite{titsias2009variational, hensman2013gaussian} to the posterior.
Here, we approximate the inputs $\mathbf{x}$ and latents $\mathbf{w}$ with $M<N$ `inducing' inputs $\mathbf{z}$, and their corresponding outputs with $\mathbf{u}$.
This formulation now has a cost of $\mathcal{O}(M^3)$.

We collectively denote all hyperparameters of the model with $\bm{\lambda}$.
The latent variable Gaussian process for modelling $p(\mathbf{y}|\mathbf{x}, \bm{\lambda})$ has the following form
\begin{align}
    &\mathbf{y}|\mathbf{x}, \mathbf{w} \sim \mathcal{N}(\mathbf{f}(\mathbf{x}, \mathbf{w}), \sigma^2), \ \ \ \ \mathbf{F}|\mathbf{u} \sim \mathcal{N}(\mathbf{K}_{\mathbf{fu}}\mathbf{K}_{\mathbf{uu}}^{-1}\mathbf{u}, \mathbf{\Sigma}), \nonumber \\
    & \mathbf{u} \sim \mathcal{N}(\mathbf{0}, \mathbf{K}_{\mathbf{uu}}),  \ \ \ \ \  \ \ \  \ \ \ \ \ \ \ \ \ \ \ [\mathbf{K}_{\mathbf{ff}}]_{nn'} = \mathbf{K}(x_n, x_{n'}), \nonumber
\end{align}
where $\Sigma = \mathbf{K}_{\mathbf{ff}} -\mathbf{K}_{\mathbf{fu}} \mathbf{K}_{\mathbf{uu}}^{-1}\mathbf{K}_{\mathbf{uf}}$, and $\sigma$ denotes the likelihood noise hyperparameter.
We posit an approximate posterior over the latents $q(\mathbf{w})$ and inducing outputs $q(\mathbf{u})$ following the variational inference framework.
This yields a lower bound to the log marginal likelihood that can then be maximised with respect to the variational distributions $q$ and the inducing inputs $\mathbf{z}$.
The lower bound for the conditional $p(\mathbf{y}| \mathbf{x}, \bm{\lambda})$ has the form
\begin{align}
     \mathcal{L}_{\mathbf{y}| \mathbf{x}}(q, \mathbf{z}, \bm{\lambda}) \!\! := \!\! \sum_n \mathbb{E}_{q(w_n)q(\mathbf{u})p(\mathbf{f}|\mathbf{u})}\left[ \log p(y_n| x_n, w_n) \right] 
    \! + \! KL[q(\mathbf{w})|| p(\mathbf{w})] \! +\! KL[q(\mathbf{u})|| p(\mathbf{u})].
    \label{eq:lower_bound_conditional}
\end{align}
For smaller datasets, we follow the steps in \citet{titsias2010bayesian} and analytically integrate over $p(\mathbf{f}|\mathbf{u})$ first and then find the closed form solution for $q(\mathbf{u})$ --- denoted GPLVM-closed form.
For certain kernels, for example the RBF and linear kernels, we can analytically find the expectation under $q(w_n)$ of the remaining first term in  \cref{eq:lower_bound_conditional}.
For larger datasets, we need to use stochastic variational inference \cite{hensman2013gaussian} as finding the closed form solutions are prohibitive --- we denote this GPLVM-stochastic.
This requires using doubly stochastic variational inference to calculate the expectations \cite{lalchand2022generalised}.
We assume a standard Gaussian distribution for $p(\mathbf{w})$ and a Gaussian for $q(\mathbf{w})$, with variational parameters to be trained.
The final KL term is between two Gaussians and is analytically tractable.
We use the evidence approximation for the hyperparameters, which we detail in  \cref{sec:evidence_approx}. 

The bound for all the marginal and conditional models, $p(\mathbf{y}| \mathbf{x})$, $p( \mathbf{x})$, $p(\mathbf{x}| \mathbf{y})$, and $p(\mathbf{y})$, follows the form of  \cref{eq:lower_bound_conditional}.
We assume the same kernels and priors for all the models as well.
    
\subsection{Final Score}

To model each distribution, we maximise the lower bound with respect to the variational distributions $q$ and inducing inputs $\mathbf{z}$ to tighten the bound.
Simultaneously, we maximise the lower bound with respect to the kernel hyperparameters and the likelihood noise, collectively denoted as $\bm{\lambda}$ (see \cref{sec:evidence_approx}).
The final scores we calculate are
\begin{align}
    \mathcal{F}_{X \rightarrow Y} &= \mathcal{L}_{\mathbf{x}}(\hat{q}, \hat{\mathbf{z}}, \hat{\bm{\lambda}}) + \mathcal{L}_{\mathbf{y}|\mathbf{x}}(\hat{q}, \hat{\mathbf{z}}, \hat{\bm{\lambda}}), \\ 
    \mathcal{F}_{Y \rightarrow X} &= \mathcal{L}_{\mathbf{y}}(\hat{q}, \hat{\mathbf{z}}, \hat{\bm{\lambda}}) + \mathcal{L}_{\mathbf{x}|\mathbf{y}}(\hat{q}, \hat{\mathbf{z}}, \hat{\bm{\lambda}}),
\end{align}
where $(\hat{q}, \hat{\mathbf{z}}, \hat{\bm{\lambda}})$ denote the values the maximise the corresponding lower bound.
We finally infer the predicted causal model as $\mathcal{M}_{X \rightarrow Y}$ if $\mathcal{F}_{X \rightarrow Y} > \mathcal{F}_{Y \rightarrow X}$, $\mathcal{M}_{Y \rightarrow X}$ if $\mathcal{F}_{Y \rightarrow X} > \mathcal{F}_{X \rightarrow Y}$, and undecided otherwise.

\section{Justifying MAP estimation of Hyperparameters}
\label{sec:evidence_approx}

In this section, we give details on the approximation we consider to integrate over the prior over hyperparameters.
This is standard practice in Gaussian process training \cite{titsias2009variational, damianou2013deep, dutordoir2018gaussian, rasmussen2003gaussian}.

We need to integrate out all model hyperparameters to get an accurate value of the marginal likelihood.
Otherwise, the actual quantity being compared is the posterior of the model with a specific hyperparameter setting.
Due to non-linearity of kernels, the integral over priors over hyperparameters tend to be intractable for our method, the GPLVM.
As such, we use the Laplace Approximation to approximate these integrals  \cite{mackay1999comparison}.
Taking the conditional $p(\mathbf{y}| \mathbf{x})$ as an example (leaving out terms in the conditional for simplicity), we wish to calculate 
\begin{align}
    p(\mathbf{y}| \mathbf{x}) = \int p(\mathbf{y}| \mathbf{x}, \lambda) p(\bm{\lambda}) d\bm{\lambda},
    \label{eq:evidence_integral}
\end{align}
where $p(\bm{\lambda})$ is a prior over the hyperparameters.
The justification for this approximation is that integral in  \cref{eq:evidence_integral} is simply the normalising constant of the posterior over the hyperparameters.
This posterior tends to be highly peaked, and even more so as the number of datapoints increases and the number of hyperparameters are few \cite{rasmussen2003gaussian}.
Thus, as most of the volume is around the MAP solution of the posterior, we can assume a Gaussian distribution around this point and approximate the integral \cref{eq:evidence_integral} as the normalising constant of this Gaussian.
The approximation of \cref{eq:evidence_integral} is 
\begin{align}
    \log p(\mathbf{y}| \mathbf{x}) &\approx \log p(\mathbf{y}| \mathbf{x}, \hat{\bm{\lambda}})p(\hat{\bm{\lambda}}) \left| \frac{1}{2 \pi} \mathbf{A}  \right|^{-\frac{1}{2}},  \\
    \hat{\bm{\lambda}} &= \argmax_{\bm{\lambda}} p(\mathbf{y}| \mathbf{x}, \bm{\lambda}), \\
    \mathbf{A} &= - \nabla^2_{\bm{\lambda}} \log p(\mathbf{y}, \bm{\lambda}| \mathbf{x}).
\end{align}
We ignore the value $p(\hat{\bm{\lambda}})$ as we don't actually take a prior over the hyperparameters, this can be thought of assuming the same density over all hyperparameter values \cite{rasmussen2003gaussian}.
We further ignore $\mathbf{A}$; this is based on the expectation that in the large sample limit, the posterior of the hyperparameters will concentrate around a single point.
We can thus safely ignore $\mathbf{A}$, and simply approximate $\log P(\mathbf{y}| \mathbf{x}) \approx \log P(\mathbf{y}| \mathbf{x}, \hat{\bm{\lambda}})$.
We find that this works well in practice.

It is possible to `overfit' with this approximation \cite{ober2021promises}.
This is not a major issue in our chosen model, as the number of hyperparameters is very low compared to the number of data samples.

As we lower bound $P(\mathbf{y}| \mathbf{x}, \bm{\lambda})$ using $\mathcal{L}_{\mathbf{y}|\mathbf{x}}$ in  \cref{eq:lower_bound_conditional}, we can carry out the procedure described above by considering  $\mathcal{L}_{\mathbf{y}|\mathbf{x}}$ instead.
Thus, our approximation of the integral over the hyperparameters will involve finding the values of the hyperparameters $\hat{\bm{\lambda}}$ that maximise $\mathcal{L}_{\mathbf{y}|\mathbf{x}}$ and using $\mathcal{L}_{\mathbf{y}|\mathbf{x}}(q, \hat{\bm{\lambda}})$.

\section{Experiment Details}
We outline the experimental details of our method. As we implement SLOPPY, we also outline the details of the settings we used.
The results for the rest of the baselines were taken from \cite{guyon2019evaluation} and \cite{duong2021bivariate}.

\subsection{Dataset details}
\label{sec:dataset_details}

\textbf{CE-Cha}: A mixture of synthetic and real world data. Taken from the cause-effect pairs challenge \cite{guyon2019evaluation}.

\textbf{CE-Multi} \cite{goudet2018learning}: Synthetic data with effects generated with varying noise relationships. The noise relationships are pre-additive $(f(X + E))$, post-additive $(f(X) + E)$, pre-multiplicative $(f(X \times E))$, or post-multiplicative $(f(X) \times E)$. The function is linear or polynomial.

\textbf{CE-Net} \cite{goudet2018learning}: Synthetic data with randomly initialised neural networks for functions and random exponential family distributions chosen for the cause.

\textbf{CE-Gauss} \cite{mooij2016distinguishing}: Synthetic data generated with random noise distributions  $ E_1, E_2$ defined in \cite{mooij2016distinguishing}.
The cause and effect are generated according to $X = f_x(E_1)$ and $Y = f_y(X, E_2)$, where $f_x, f_y$ are sampled from Gaussian processes. 

\textbf{CE-Tueb} \cite{mooij2016distinguishing}: Contains 105 pairs of real cause effect pairs taken from the UCI dataset. We use the version dating August 22, 2016.
We also remove high dimensional datasets leaving 99 datasets in total.

\subsection{GPLVM details}
\label{sec:gplvm_details}

We use the GPLVM-closed form for all datasets except for CE-Tueb where we GPLVM-stochastic due to the high number of variables in some of the datasets.

For GPLVM-closed form, we use the sum of an RBF and linear kernels.
$q(w_n)$ has an analytical expectation for these kernels, as discussed in  \cref{sec:inducing_point_gplvm}.
As detailed in  \cref{sec:inducing_point_gplvm}, we find the optimal form of $q(\bm{u})$ following the procedure of \cite{titsias2010bayesian}.
We use 200 inducing points for all experiments.
The model was first trained using Adam with a learning rate of 0.1. 
After 20,000 epochs, the model was trained using BFGS.
We found that this greatly helped the numerical instability of BFGS, but found better ELBO (variational approximation to the marginal likelihood) values than simply using Adam.

For GPLVM-stochastic, as expectations are calculated by sampling, it was possible to use a larger number of kernels.
We used a sum of RBF, Linear, Matern32 and Rational Quadratic kernels.
10 samples were used to calculate the expectations.
The model was trained with Adam with a learning rate of 0.05.
The model stopped training if the value of the ELBO plateaued, else it ran for a maximum of 100,000 epochs.
In our experiments, we only use GPLVM-stochastic for CE-Tueb  as it had a few datasets that had a large number of samples.

The approximate posteriors for both the models $q (w_n)$ are initialised with low variance and the mean equal to $0.01$ times the output.
This reduced the instability during optimisation.

As GPLVMs are known to suffer from local optima issues, we use 20 random restarts of hyperparameter initialisations, and choose the highest estimate of the approximate marginal likelihood as the final score.
For the various hyperparameters, the sampling procedures were:
\begin{enumerate}
    \item The kernel variances were always set to 1. 
    \item The likelihood variances were sampled by first sampling $\kappa \sim \text{Uniform}(10, 100)$, and then $\sigma^2_{\text{Likelihood}} = 1/\kappa^2$.
    \item The kernel lengthscales were sampled by first sampling $\psi \sim \text{Uniform}(1,100)$, then set $\lambda_{\text{Lengthscale}} = 1/\psi$.
\end{enumerate}

\subsection{SLOPPY details}
\label{sec:sloppy_details}

For benchmarking the SLOPPY method, we use the author's code  \cite{marx2019identifiability}. We use  the spline estimator as it performs better on all the dataset. 
For this estimator, we select the best performing regularisation metric between the AIC and BIC.

\section{Additional Experiments}
\label{sec:additional_exp}

We carry out some additional experiments that give us insight into our method. 
In \cref{sec:anm_data_results}, we show that the GPLVM performs well on ANM data, despite being more fleixble than an ANM.
In \cref{sec:only_conditional_results}, we show that importance of modelling the joint instead of just the conditional or marginal densities.

\subsection{ANM Data}
\label{sec:anm_data_results}

\begin{table}[h!]
\centering
\caption{ROC AUC scores for identifying causal direction of datastes generated by an ANM (higher is better).}
\begin{tabular}{ll}
 \toprule
 Methods &  ANM\\
\midrule
   Gaussian Process & 100.0 \\
   GPLVM & 100.0 \\
 \bottomrule
\end{tabular}
\label{tab:anm_results}
\end{table} 

ANM is an example of a strictly identifiable model. Here we show
that our added flexibility does not result in a loss in performance when compared to identifiable models.

The GPLVM model contains the hyperparameters $\boldsymbol{\lambda}$ of the GP priors, which makes the GPLVM  a hierarchical model. 
Depending on which $\boldsymbol\lambda$ is inferred, the GPLVM can learn to behave in different ways. 
For example, for datasets that follow ANM assumptions, the effect of the latent variable $w_i$ can be ignored, making the model behave as an ANM. 
In this section, we show that the added flexibility of the GPLVM model (over ANM) does not lead to a reduction in performance when tested on data from an ANM.
 
We use datasets generated from an ANM (taken from \citet{tagasovska2020distinguishing}). 
A straightforward Gaussian process (GP) model satisfies the conditions of an ANM which have been shown to identify causal direction using the likelihood only \cite{zhang2015estimation}. 
\Cref{tab:anm_results} shows that the marginal likelihood also perfectly identifies causal direction.
Furthermore, even though a GPLVM is a more flexible model than a GP, the added flexibility does not suffer from a loss of performance.

\subsection{Only modelling the conditional or marginal}
\label{sec:only_conditional_results}
Methods such that ANM, PNL, SLOPPY, RECI only model the conditional densities to find the causal direction. 
In fact,  methods such as SLOPPY base their theory on modelling the joint, but make the assumption that the cause is always Gaussian distributed, and hence only consider the conditional.
The Kolmogorov complexity formalisation of the ICM principle \cite{janzing2010causal, peters2016causal} also considers the whole joint.

We show that modelling the joint is crucial to our approach, and that we suffer a degradation in performance when only considering one component - the marginal or conditional.
In  \cref{tab:appendix_conditional}, we show that results of the same model, but making the decision on the predicted causal model with the joint, the conditional, or with the marginal.
The results corroborate with our theory.

\begin{table*}[h!]
\centering
\caption{Results of making the decision on the predicted causal model with the full joint, or just with the marginal or conditional densities. In accordance with our theory, modelling the joint is important. The numbers are ROCAUC (higher is better).}
\begin{tabular}{llllll}
 \toprule
 Methods &   CE-Cha & CE-Multi & CE-Net & CE-Gauss & CE-Tueb\\
\midrule
   GPLVM - Joint & 81.9 & 97.7 & 98.9 & 89.3 & 78.3 \\
   GPLVM - Conditional & 61.5 & 89.7 & 70.3 & 21.7 & 36.2 \\
   GPLVM - Marginal & 44.5 & 23.3 & 42.6 & 83.1 & 75.7 \\
 \bottomrule
\end{tabular}
\label{tab:appendix_conditional}
\end{table*}

\end{document}